\documentclass[preprint,12pt]{elsarticle}

\usepackage[numbers]{natbib}
\usepackage{amssymb}
\usepackage{graphicx}
\usepackage{amsmath}
\usepackage{bm}
\usepackage{booktabs}
\usepackage{multirow}
\usepackage{times}
\usepackage{epsfig}
\usepackage{algorithm}
\usepackage{algorithmic}
\usepackage{enumitem}
\usepackage[flushleft]{threeparttable}
\usepackage{subcaption}
\usepackage{setspace}
\usepackage{color}
\usepackage{enumitem}

\usepackage{hyperref}

\def\ie{i.e.}
\def\eg{e.g.}
\def\st{\textrm{s.t.}}
\def\and{\textrm{and}}

\def\dim{\textrm{dim}}

\def\Diag{\textrm{Diag}}

\def\trace{\textrm{trace}}
\def\Diag{\textrm{Diag}}

\def\0{\textbf{0}}
\def\1{\textbf{1}}

\def\g{\boldsymbol{g}}
\def\h{\boldsymbol{h}}

\def\q{\boldsymbol{q}}

\def\x{\boldsymbol{x}}

\def\z{\boldsymbol{z}}

\def\l{\boldsymbol{l}}

\def\EE{\mathbb{E}}

\def\G{\mathcal{G}}

\def\L{\mathcal{L}}

\def\V{\mathcal{V}}

\def\RR{\mathbb{R}}

\newcommand{\myparagraph}[1]{\noindent\textbf{#1.}}

\newtheorem{theorem}{Theorem}

\newtheorem{proof}{Proof}

\usepackage[capitalize]{cleveref}
\crefname{section}{Sec.}{Secs.}
\Crefname{section}{Section}{Sections}
\Crefname{table}{Table}{Tables}
\crefname{table}{Tab.}{Tabs.}

\newcommand{\mcb}{\color{black}}

\journal{Pattern Recognition}

\begin{document}

\begin{frontmatter}

\title{Neural Normalized Cut: A Differential and Generalizable Approach for Spectral Clustering}

\author[1]{Wei~He}

\author[1]{Shangzhi~Zhang}

\author[1]{Chun-Guang~Li}

\author[2]{Xianbiao~Qi}

\author[2]{Rong~Xiao}

\author[1]{Jun~Guo}

\address[1]
{{School of Artificial Intelligence},{~Beijing University of Posts and Telecommunications},{~Beijing},{~P.R. China}}

\address[2]
{{Intellifusion},{~Shenzhen},{~P.R. China}}

\begin{abstract}
Spectral clustering, as a popular tool for data clustering, requires an eigen-decomposition step on a given affinity to obtain the spectral embedding.
Nevertheless, such a step suffers from the lack of generalizability and scalability.
Moreover, the obtained spectral embeddings can hardly provide a good approximation to the ground-truth partition and thus a $k$-means step is adopted to quantize the embedding. 
In this paper, we propose a simple yet effective scalable and generalizable approach, called Neural Normalized Cut (NeuNcut), to learn the clustering membership for spectral clustering directly.
In NeuNcut, we properly reparameterize the unknown cluster membership via a neural network, and train the neural network via stochastic gradient descent with a properly relaxed normalized cut loss.
As a result, our NeuNcut enjoys a desired generalization ability to directly infer clustering membership for out-of-sample unseen data and hence brings us an efficient way to handle clustering task with ultra large-scale data.
We conduct extensive experiments on both synthetic data and benchmark datasets and experimental results validate the effectiveness and the superiority of our approach.

\end{abstract}


\end{frontmatter}

\section{Introduction}
\label{sec:intro}

Finding clusters from unlabeled data is a task of great scientific significance and practical value in pattern recognition, machine learning and data science.
Spectral clustering~\cite{vonLuxburg:StatComp07, Filippone:PR2008-Survey}, as a popular tool for data clustering, has widely spread applications in variant areas, owning to its wide generality, excellent empirical performance and rich theoretical foundation from spectral graph theory~\cite{Ding:PR2024-graph_survey}.
Spectral clustering finds a graph partition that minimizes the sum of weights on edges between different subgraphs---which is essentially solving a problem of minimizing a graph cut. Due to its combinatorial nature, rather than solving the problem directly, as in, \eg, Ratio cut \cite{Chan:IDAC1993-Rcut}, Normalized cut (Ncut) \cite{Shi:IEEE2000-Ncut} and Min-max cut \cite{Ding:ICDM2001-Minmax}, 
the common practice in spectral clustering is to solve a relaxed problem and find the partition in graph spectral domain.

Roughly, spectral clustering algorithms consist of two main steps:
1) computing the spectral embeddings via the eigenvectors associated with the ending $k$ minor eigenvalues of the pre-computed graph Laplacian; and then
2) applying $k$-means algorithm to obtain the partition from the embedding.
Despite the popularity of spectral clustering, it suffers from various drawbacks.
The spectral embeddings are usually computed only for the given samples, thus they cannot be generalized to unseen out-of-sample data.
Moreover, the eigen-decomposition of the graph Laplacian that is of quadratic complexity with number of samples is also quite time consuming or computationally prohibitive when handling clustering task with ultra large-scale data.
Besides, the spectral embeddings obtained by eigen-decomposition contain negative entries, hence can hardly provide a good approximation to the clustering membership. 
As a remedy, a $k$-means algorithm usually is applied on the spectral embeddings as a rounding heuristic step to obtain the clustering membership. Since that the problem is solved in the two separated stages, the solution obtained by $k$-means is sub-optimal to the original graph cut problem~\cite{Zhu:PR2020-one}.

To address the scalability issue, methods based on Nystr{\"o}m extension~\cite{Fowlkes:2004-Nystrom}, landmarks~\cite{Yang:PR2023-RESKM} and bipartite graphs~\cite{Yang:PR2020-Fast} have been proposed.
Still, none of these methods can generalize the spectral embeddings.
Recently, there are a few attempts to tackle the scalability and generalizability for spectral embeddings, \eg, SpectralNet~\cite{Shaham:ICLR18} and SpecNet2~\cite{Chen:MSML2022-SpecNet2} train neural networks to approximate the eigenvectors of the graph Laplacian.
Nevertheless, as in conventional spectral clustering, these methods still have to apply a $k$-means step on the learned spectral embeddings, thus the obtained solution is sub-optimal to the original graph cut problem.

In this paper, we attempt to develop an efficient and effective approach to learn the clustering membership for spectral clustering directly, aiming to endow spectral clustering with the generalization ability to handle out-of-sample data.
Specifically, we reformulate the problem of minimizing the normalized graph cut loss by incorporating a relaxed orthogonality penalty and a set of properly adjusted box constraints on the continuous relaxation of segmentation matrix at first;
then we reparameterize the segmentation matrix via a neural network with a softmax output to learn the soft clustering membership (rather than the spectral embedding). Because of employing the normalized cut loss and incorporating the neural network to reparameterize the clustering membership, we term our approach as a {\em Neural Normalized Cut} ({\bf NeuNcut}).
By relaxing the orthogonality constraint, our NeuNcut is easier to train. 
Owning to introducing a neural network with a softmax output, our NeuNcut can directly infer the clustering membership with no need to use the $k$-means step to quantize the embeddings.
Our code is available at: \url{https://github.com/hewei98/NeuNcut}.

The contributions of the paper can be summarized as follows.
\begin{enumerate}
    
    \item We reformulate the combinatorial problem of minimizing a normalized cut into continuous problem with a relaxed orthogonality penalty and a set of properly adjusted box constraints. The box constraints enable the segmentation matrix to satisfy more desired properties. 

    \item We propose to reparameterize the clustering membership matrix with a properly designed neural network and therein present an EM-like procedure to solve the reformulated normalized cut problem.
    
    \item We conduct extensive experiments on both synthetic data and real world data to demonstrate the effectiveness of our approach, as well as the generalization ability to out-of-sample data.
\end{enumerate}

\section{Related Work}
\label{sec:relate}

\subsection{Spectral Clustering}
Spectral clustering is a well-known approach to approximate the solution for graph cut problem. Instead of directly solving the problem of combinatorial nature, the common practice relaxes the combinatorial graph cut problem into its corresponding continuous problem at first and solves it as an eigen-decomposition problem with the graph Laplacian instead. 
Different spectral clustering methods differ in the ways to avoid trivial solutions. 
For instance, Ratio Cut~\cite{Chan:IDAC1993-Rcut} maximizes the number of data points within each partition; Normalized Cut (Ncut) \cite{Shi:IEEE2000-Ncut} maximizes the number of edges within each partition; and Min-max Cut~\cite{Ding:ICDM2001-Minmax} maximizes the similarity within each partition. 
Besides, by connecting to nonnegative matrix factorization~\cite{Boutsidis:PR2008-SVD}, variant nonnegative spectral clustering methods~\cite{Lu:PR2014-Non, Shang:PR2016-Global} have also been developed, in which the nonnegativity constraint is incorporated to enforce the nonnegativity of the segmentation matrix.

\subsection{Scalable Spectral Clustering}
Spectral clustering suffers from a scalability issue when handling large-scale data due to its heavy memory requirements for storing pairwise affinities and computational cost for computing the spectral embedding. 
In the earlier stage, Nystr{\"o}m method \cite{Fowlkes:2004-Nystrom} is employed to address the scalability issue by randomly choosing a few samples to construct the affinity sub-matrix.
In \cite{Yang:PR2023-RESKM}, a unified framework for landmark-based spectral clustering is presented, in which an anchor graph is constructed and what follows is eigen-decomposition, and $k$-means is performed on the anchor space.
Additionally, when a sparse affinity sub-matrix is obtained, bipartite graph clustering methods can be used to obtain the partition, \eg, \cite{Yang:PR2020-Fast}. 
A coarse-to-fine anchor approximation strategy is proposed in \cite{Huang:IEEE2019-Ultra} to further reduce the computation cost for dealing with ultra large-scale dataset.
Nevertheless, all these methods mentioned above are performed on the given sample and lack the ability to handle out-of-sample unseen data.

\subsection{Deep Clustering}
Inspired by the success of deep learning in the past decade, a number of methods attempt to learn the clustering by leveraging a deep learning framework.
For example, in VaDE \cite{Jiang:IJCAI2016-VaDE}, a variational autoencoder equipped with a Gaussian mixture prior is trained to perform deep embedding and clustering;
in DEPICT \cite{Ghasedi:ICCV17-DEPICT}, a pre-trained auto-encoder with a softmax layer is trained with a relative entropy based loss, in which the target distribution is initialized with clustering method; 
in SCAN \cite{Van:ECCV2020-SCAN}, a novel framework based on minimizing the consistency between pairwise samples is presented, in which a regularization based on maximizing the entropy of the clustering assignments \cite{Caron:ECCV2018} is incorporated to prevent the clustering degeneracy.
However, these methods such as DEPICT and SCAN utilize an entropy-maximizing regularization to prevent collapsed solution.
Adopting an entropy-maximizing regularization is essentially assuming that all clusters are of equal sizes, which is misleading when the clusters in data are imbalanced.

\subsection{Differential Spectral Clustering}
Recently, there are a few attempts to address the limitations of spectral clustering in terms of scalability and generalization ability.
SpectralNet \cite{Shaham:ICLR18} introduces deep neural networks with an orthogonalization layer to learn the spectral embeddings and then uses $k$-means to quantize the embedding.
SpecNet2~\cite{Chen:MSML2022-SpecNet2} proposes an orthogonalization-free objective to learn the spectral embedding.
AutoSC~\cite{Fan:NIPS2022-AutoSC} integrate an automatically constructed affinity matrix with a neural network to learn the spectral embedding.
BaSiS~\cite{Streicher:CVPR2023-BaSiS} learns the spectral embeddings by using affine registration techniques to align the mini-batches.
These methods, SpectralNet, SpecNet2, AutoSC and BaSiS, are designed to learn the orthogonal spectral embeddings of the graph Laplacian and thus requires to use $k$-means on the learned embedding of the entire dataset to find the clusters, leading to unsatisfactory clustering results.
CNC \cite{Nazi:arXiv2019-CNC} aims to directly optimize the normalized cut objective via a neural network without continuous relaxation and explicit orthogonalization constraint.
However, since that the loss function of CNC is still of combinatorial nature, it is quite challenging to optimize it without proper relaxation.
 Rather than enforcing a strict orthogonality constraint to learn the spectral embeddings (as in \cite{Shaham:ICLR18}) or directly optimizing a normalized cut loss of combinatorial nature (as in \cite{Nazi:arXiv2019-CNC}), we take a compromise path to directly learn the clustering membership with a properly relaxed normalized cut loss.

\section{Our Approach: Neural Normalized Cut}
\label{sec:method}

This section will introduce some preliminaries in spectral clustering with normalized cut at first and then reformulate the normalized cut problem and present our approach---Neural Normalized Cut (NeuNcut).

\subsection{Preliminaries in Normalized Cut and Its Relaxations}
\label{sec:preliminary}

Spectral clustering is able to handle non-convex clusters. However, solving the spectral clustering simply by minimizing the vanilla graph cut objective will result in trivial solution, that consists of singleton clusters. 
In one of the most popular spectral clustering methods, normalized cut \cite{Shi:IEEE2000-Ncut}, the trivial solution is prevented by taking into account the volume of clusters.

Given a dataset of $n$ data points $\x_i \in \RR^d$, 
arranged as a data matrix $X=[\x_1,\cdots, \x_n] \in \RR^{d \times n}$. Denote a given affinity matrix as $A \in \mathbb R^{n\times n}$, in which each element $a_{i,j}$ measures the pair-wise similarity between data points $\x_i$ and $\x_j$.
By viewing each data point $\x_i$ as graph vertex $v_i$, we can define a graph $\G(\V, A)$, where $\V$ is the set of vertexes. The goal of spectral clustering is to find the partition $\{\V^{(1)},\ldots , \V^{(k)}\}$ for the vertexes in $\V$ on the graph, where $\V = \V^{(1)} \cup \V^{(2)} \cup \cdots \cup \V^{(k)}$ and $k$ denotes the number of true clusters. Precisely, the normalized cut objective is defined as follows:
%
\begin{equation}
\label{eq:ncut1}
\mathrm{Ncut}(\V^{(1)},\ldots,\V^{(k)}) := \sum_{\ell=1}^k \frac{\mathrm{cut}(\V^{(\ell)},\overline{\V}^{(\ell)})}{\mathrm{vol}(\V^{(\ell)})},
\end{equation}
%
where $\overline{\V}^{(\ell)}$ denotes the complement of $\V^{(\ell)}$, $\mathrm{cut}(\V^{(\ell)},\overline{\V}^{(\ell)})=\sum_{i: v_i \in {\V}^{(\ell)}} \sum_{j: v_j \notin {\V}^{(\ell)}}a_{i,j}$, $\mathrm{vol}(\V^{(\ell)})=\sum_{i: v_i \in \V^{(\ell)}}D_{i,i}$ is the volume of the $\ell$-th cluster, and $D_{i,i}=\sum_{j=1}^n a_{i,j}$ is the degree of the $i$-th vertex.

Let $H = [\h_1, \cdots, \h_k] \in \{0,1\}^{n \times k}$ be a binary segmentation matrix, for which $\h_\ell={(h_{1,\ell},\cdots,h_{n,\ell})}^\top$, where nonzero entries indicate which vertexes belong to the $\ell$-th cluster, \ie,
\begin{equation} 
\label{eq:indicator}
h_{i,\ell} :=
\begin{cases}
1& \ \text{if} \ {v}_i \in \V^{(\ell)}\\
0& \ \text{if} \ {v}_i \notin \V^{(\ell)}.
\end{cases}
\end{equation}
Then, the normalized cut objective can be rewritten as follows (see, \eg, \cite{vonLuxburg:StatComp07} for a detailed deduction):
\begin{equation}
\label{eq:proof_trace_loss2}
    \mathrm{Ncut}(\V^{(1)},...,\V^{(k)})=\sum_{\ell=1}^k \tilde\h_\ell^\top L \tilde \h_\ell = \trace{(\tilde H^\top L \tilde H)},
\end{equation}
where $\tilde H=[\tilde \h_1, \cdots \tilde \h_k]$, in which $\tilde \h_\ell = \frac{1}{\sqrt{\mathrm{vol}(\V^{(\ell)})}}\h_\ell $, $L = D - A$ is the graph Laplacian, and $D=\Diag(D_{1,1},\cdots , D_{n,n})$ is an $n \times n$ degree matrix.

\myparagraph{Relaxing Segmentation Matrix Continuously}
Solving for the partition $\{\V^{(1)},\ldots \V^{(k)}\}$ turns out to be 
finding for the segmentation $H$.
Unfortunately, finding $\{\h_\ell\}_{\ell=1}^k$ for $k > 2$ is an NP-hard combinatorial optimization problem.
Alternatively, the common practice in spectral clustering is to solve a continuous relaxation of $\tilde H$. For the Ncut problem, one solves: 
\begin{align}
\label{eq:ncut2}
\min_{\tilde H \in \mathbb R^{n \times k}} \, \trace{(\tilde H^\top L \tilde H)} \quad \st \quad \tilde H^\top D \tilde H = I,
\end{align}
where $I\in \mathbb R^{k\times k} $ is an identity matrix and the constraint $\tilde H^\top D \tilde H = I$ is to avoid trivial solution.
By letting $F = D^{1/2} \tilde H$, the objective in Eq.~\eqref{eq:ncut2} becomes:
\begin{align}
\label{eq:ncut3}
\min_{F \in \mathbb R^{n \times k}} \, \trace{(F^\top \Tilde{L} F)} \quad \st \quad F^\top F = I,
\end{align}
where $\Tilde{L} = D^{-\frac{1}{2}} L D^{-\frac{1}{2}}$ is the normalized graph Laplacian. Then it is easy to show that the matrix $F$ can be solved by computing the ending $k$ eigenvectors associated with the 
$k$ minor eigenvalues of $\tilde{L}$.

\myparagraph{Nonnegative Spectral Clustering}
While relaxing the segmentation matrix $H$ and so for $\tilde H$ into the real value matrix subject to orthogonality constraint $\tilde H^\top D \tilde H=I$ will lead to an eigenvalue decomposition problem, the obtained solution returned by the ending $k$ eigenvectors usually cannot provide a satisfactory approximation to the desired segmentation matrix. Because, the optimal solution for $\tilde H$ obtained by the ending $k$ eigenvectors may contain arbitrary negative entries which would make it deviate from being clustering membership.
To remedy such a deficiency, alternatively, nonnegative spectral clustering is developed, in which the nonnegativity constraint is imposed into the problem as follows: 
\begin{equation}
\label{eq:nonneg-ncut2}
\min_{\tilde H \in \mathbb R^{n \times k}}\, \trace{(\tilde H^\top L \tilde H)}, \quad \st ~~ \tilde H^\top D \tilde H \approx I,\ \tilde H \geq 0.
\end{equation}
However, as a cost of imposing the nonnegative constraint, the problem can no longer be solved as an eigenvalue decomposition problem, but needs to consult to 
nonnegative matrix factorization scheme, 
which makes the orthogonality constraint hardly be satisfied strictly.

\myparagraph{Remark 1}
We notice that there exists a latent trade-off in the stage of relaxing the segmentation matrix $H$ defined in~\eqref{eq:indicator} for spectral clustering.
The spectral clustering methods introduce an orthogonality constraint but at a cost of giving up the nonnegativity constraint;
whereas the nonnegative spectral clustering methods impose the nonnegativity constraint but at a cost of giving up the orthogonality constraint.
This trade-off between orthogonality and nonnegativity inspires us to take a compromise path to reformulate the normalized graph cut problem.

\subsection{Reformulating Normalized Cut Problem}
\label{sec:Reformulate NCut}
Now we begin to reformulate the normalized cut problem.
Since that we want to learn the (soft) clustering membership directly, rather than merely the spectral embedding, it is desirable that the index of the largest entry in each row of the matrix $H$ indicates the clustering membership, by which the assignment of data point to each cluster can thus be determined with no need to use an extra $k$-means step. 
Moreover, the insights from \textbf{Remark 1}
hint us that the orthogonality constraint is not necessary as long as additional constraints are imposed to harness $H$ to satisfy the desired property for being clustering membership.
Thus, we propose to address the normalized cut problem by solving the following relaxed problem: 
\begin{equation}
\begin{split}
\label{eq:our-nonneg-ncut-orth-penalty}
    \min_{\tilde H \in \mathbb R^{n \times k} } &\  \trace \left (\tilde H^\top L \tilde H \right )  + \frac{\gamma}{2} \left \|\tilde H^\top D \tilde H -I \right \|_F^2, \\
    \quad \st &  ~~~~ 0\leq \tilde H \Lambda \leq 1 ,\ ( \tilde H \Lambda) \cdot \mathbf{1} = \mathbf{1},
\end{split}
\end{equation}
where $\gamma >0$ is a penalty parameter, $\|\tilde H^\top D \tilde H - I \|_F^2$ is a relaxed orthogonality constraint, $\mathbf{1}$ is a $k$-dimensional vector consisting of $1$'s, and $\Lambda \in \RR^{k \times k}$ is an unknown diagonal matrix defined by:
\begin{equation}
    \label{eq:true_volume}
    \Lambda = \mathrm{Diag} \left( \mathrm{vol}(\V^{(1)}), \cdots, \mathrm{vol}(\V^{(k)}) \right)^{\frac{1}{2}}.
\end{equation}
Compared to the conventional normalized cut formulation in Eq.~\eqref{eq:ncut2} and the nonnegative spectral clustering in Eq.~\eqref{eq:ncut3}, the main differences are two-folds: a) the orthogonality constraint $\tilde H^\top D \tilde H = I$ is relaxed; and b) a set of adjusted box constraints $0 \leq \tilde H \Lambda  \leq 1 $ and $(\tilde H \Lambda) \cdot \mathbf{1} =  \mathbf{1} $ are added, rather than using the naive nonnegativity constraint $H \ge 0$.
The reason to incorporate such a set of calibrated constraints is to more accurately and properly approximate the segmentation matrix $H$ as defined in Eq.~\eqref{eq:indicator} for the relaxed problem. 
It 
will be clear soon that such constraints can be elegantly and implicitly satisfied.

Rather than solving for $\tilde H$ directly from problem~\eqref{eq:our-nonneg-ncut-orth-penalty}, by noting of $\tilde H = H \Lambda^{-1}$, 
we reparameterize $H$ via a Multi-Layer Perceptron (MLP) with a softmax output layer, \ie,
\begin{equation}
    \label{eq:forward}
    \mathbf f(\cdot; \Theta) = \texttt{softmax} ( \mathbf g(\cdot; \Theta) ),
\end{equation}
where $\mathbf g: \RR^d \rightarrow \RR^k$ is the MLP and $\Theta$ denotes all the parameters in the network.
Given the input data $X$ and, we have that $H$ is approximated by the network output $\mathbf{f}(X;\Theta)$, which can be optimized by updating the parameters in $\Theta$.
Owning to the {\texttt {softmax}} layer, the output of the network can be served to approximate the clustering memberships.
In particular, we note that by incorporating the {\texttt {softmax}} layer and rewriting $\tilde H = f(X;\Theta)\Lambda^{-1}$, 
the constraints $0\leq \tilde H \Lambda \leq 1 ,\ (\tilde H \Lambda) \cdot \mathbf{1} =  \mathbf{1} $ in Eq.~\eqref{eq:our-nonneg-ncut-orth-penalty} can be automatically satisfied, provided that $\Lambda$ was known or estimated.

The loss function to train the neural network $\mathbf f(\cdot; \Theta)$ turns out to be:
\begin{equation}
\label{eq:our_loss}
\begin{split}
    \L(X, A;\Theta) := & \underbrace{\mathrm{\trace} \left (( \mathbf f (X; \Theta) \Lambda^{-1})^\top \cdot L \cdot ( \mathbf f (X; \Theta) \Lambda^{-1}) \right )}_{\mathcal{L}_{Lap}} \\
    + &\frac{\gamma}{2} \underbrace{\left \|( \mathbf f (X; \Theta) \Lambda^{-1})^\top \cdot D \cdot (  \mathbf f (X; \Theta) \Lambda^{-1})-I \right \|_F^2}_{\mathcal{L}_{orth}}.    
\end{split}
\end{equation}
%
Unfortunately, the matrix $\Lambda$ in Eq.~\eqref{eq:our_loss} is still unknown.
Let $Y = \mathbf f (X; \Theta) \in \RR^{n \times k}$ be the output 
matrix, we replace the volume of the $\ell$-th cluster $\mathrm{vol}(\V^{(\ell)})$ by the following estimation:
\begin{equation}
\label{eq:soft_vol}
    \widetilde {\mathrm{vol}}(\V^{(\ell)}) :=  \sum_{i=1}^n y_{i,\ell} \cdot D_{i,i},
\end{equation}
where $y_{i,\ell}$, the element of the output matrix $Y$, is the belief to assign the $i$-th data point to $\ell$-th cluster $\V^{(\ell)}$.
Then, according to the definition in Eq.~\eqref{eq:true_volume}, 
$\Lambda$ also can be replaced by:
\begin{equation}
    \label{eq:estimate_lambda}
    \tilde \Lambda = \mathrm{Diag} \left( \widetilde {\mathrm{vol}}(\V^{(1)}), \cdots, \widetilde {\mathrm{vol}}(\V^{(k)}) \right)^{\frac{1}{2}}.
\end{equation}
That is, the task of finding the desired segmentation matrix $H$ defined in Eq.~\eqref{eq:indicator} turns out to be a task of training neural network $\mathbf f(\cdot;\Theta)$ via loss in Eq.~\eqref{eq:our_loss} and then updating $\Lambda$ via Eq.~\eqref{eq:estimate_lambda} alternately.
For clarity, we summarize our training procedure in Algorithm \ref{algo:EM-like}.
After training, the network $\mathbf f(\cdot, \Theta)$ will learn the map the input data $X$ in the feature space $\RR^d$ onto the clusters assignment in space $\RR^k$.

\myparagraph{Remark 2} As an end-to-end approach, our NeuNcut can be used to replace the conventional spectral clustering for many down-stream clustering tasks, \eg, the spectral clustering step in subspace clustering methods \cite{Zhang:CVPR21-SENet}.
The benefits of our framework are two-folds.
First, the neural network $\mathbf f(X, \Theta)$ maps data points directly to their cluster memberships. Hence, it can be trained on a sampled small set of data and then generalized to infer the cluster memberships for unseen data points directly. 
This provides an efficient mechanism for handling the clustering task on ultra large-scale datasets. 
Second, the neural network can be trained in a mini-batch mode and enjoys the scalability.
{
For each mini-batch that contains $m$ data points, only a graph Laplacian of $m\times m$ is needed to compute and cache. As the number of batch samples grows, the graph Laplacian on the mini-batch data converges to the the manifold Laplacian~\cite{Belkin:NIPS06-laplacian_theory,Belkin:JCSC08-laplacian_theory}.
}

\myparagraph{Remark 3}  It is worth noting that the updating step in Eq.~\eqref{eq:estimate_lambda} is analogue to an \emph{Expectation} (E) step that estimates the expectation $\EE[\Lambda]$.
    When $\Lambda$ is fixed, we train of the neural network $\mathbf f(X; \Theta)$ by minimizing the loss function in Eq.~\eqref{eq:our_loss}---this stage is effectively analogue to a \emph{Maximization} (M) step that maximizes the likelihood $\exp{(-\mathcal{L}(X, A;\Theta))}$.
    Thus, our proposed scheme to solve NeuNCut by alternately training of the neural network $\mathbf f(X; \Theta)$ and updating $\Lambda$ is essentially an EM-style algorithm, as shown in steps~\ref{step:estimate}-\ref{step:maximize} of Algorithm \ref{algo:EM-like}.

\begin{algorithm}[t]
\small
	\caption{An EM-style Procedure for Solving NeuNcut}
	\label{algo:EM-like}
	\begin{algorithmic}[1]
		\STATE \textbf{Input:} Training data $X$, trade-off parameter $\gamma > 0$, number of iterations $T$, batch size $m$ and learning rate $\eta$.
		\STATE \textbf{Initialization:} $t=0$, random initialization of MLP parameters $\Theta^{(t)} $.
        \FOR {each $t \in \{1, \cdots, T\}$}
            \STATE \texttt{\# data preparation}
            \STATE Randomly sample mini-batch data $X^{(t)} $ from $X$.
            \STATE {Obtain affinities $ A^{(t)}  \in \RR^{m\times m}$ from $X^{(t)} $ by methods in Section~\ref{sec:aff}. }
            \STATE Compute degree $D^{(t)} $ and graph $L^{(t)} $ w.r.t $A^{(t)}$ as defined in Section~\ref{sec:preliminary}.
            \STATE \texttt{\# Forward pass}
            \STATE Compute output $Y^{(t)} \in \RR^{m\times k}$ by Eq.~\eqref{eq:forward}.
            \STATE \label{step:estimate} \texttt{\# Estimating volume}
            \STATE Compute volume $\Lambda^{(t)}  \in \RR^{k\times k}$ by Eq.~\eqref{eq:estimate_lambda}.
            \STATE \texttt{\# Backward propagation}
            \STATE Compute $\nabla_\Theta \L(X^{(t)} , A^{(t)} ; \Theta^{(t)} )$ of Eq.~\eqref{eq:our_loss}.	
            \STATE \label{step:maximize} Set $\Theta^{(t+1)}  \leftarrow \Theta^{(t)}  - \eta \cdot \nabla_\Theta \L(X^{(t)} , A^{(t)} ; \Theta^{(t)} ) $.
            \STATE $t \leftarrow t +1$.
        \ENDFOR
		\STATE \textbf{Output:} $\Theta^{(t)}$
		\end{algorithmic}
\end{algorithm}

\subsection{When Orthogonality Meets Softmax}
\label{sec:prove}

In our NeuNcut, we impose in Eq.~\eqref{eq:our-nonneg-ncut-orth-penalty} a set of adjusted box constraints which then can be automatically satisfied owning to adopting a softmax layer in the neural network $\mathbf f(\cdot; \Theta)$.
In this section, we will explain when the strict orthogonality could be satisfied provided that the neural network is used to learn $Y$. 

\begin{theorem}
    \label{theorem:1}

    Suppose $Y \in\RR^{n\times k}$ be the clustering membership matrix, produced by the softmax layer of $\mathbf f(\cdot; \Theta)$, and $\tilde H = Y \Lambda^{-1}$, where $\Lambda$ is defined in Eq.~\eqref{eq:true_volume}. Then, the orthogonality constraint $\tilde H^\top D \tilde H = I$ holds if and only if $Y$ is a binary clustering membership matrix.
\end{theorem}

\begin{proof}
    \label{proof:1}

    If the orthogonality constraints $\tilde H^\top D \tilde H = I$ is satisfied, then we must have that for the $\ell$-th cluster $\V^{(\ell)}$: $\tilde \h_\ell^\top D \tilde \h_\ell = 1$. 
    Note that
    $$\tilde \h_\ell^\top D \tilde \h_\ell = \sum_{i=1}^n D_{i,i}\cdot \tilde h_{i,\ell}^2 =\frac{\sum_{i=1}^n D_{i,i}\cdot y_{i,\ell}^2}{\mathrm{vol}(\V^{(\ell)})},$$
    where $\mathrm{vol}(\V^{(\ell)})$ is replaced by its estimation as defined in Eq.~\eqref{eq:soft_vol} in our approach. Thus we have
    $$\tilde \h_\ell^\top D \tilde \h_\ell = \frac{\sum_{i=1}^n D_{i,i} y_{i,\ell}^2 }{\sum_{i=1}^n D_{i,i} y_{i,\ell}}.$$
    Since that $0 \leq y_{i,\ell}\le 1$, we have $0 \leq y^2_{i,\ell}\le 1$ and $y^2_{i,\ell} \le y_{i,\ell}$. Note also that $D_{i,i} > 0$, thus we see that ${\sum_{i=1}^n D_{i,i} y_{i,\ell}^2 } \le {\sum_{i=1}^n D_{i,i} y_{i,\ell}}$, \ie, $\tilde \h_\ell^\top D \tilde \h_\ell \le 1$, where the equality holds only if $y_{i,\ell} = 1$ and $y_{i,\ell'} = 0$ for all $\ell' \neq \ell$. We have proved that each row of $Y$ must be binary. 
    
\end{proof}

This result tells us that if a strict orthogonality constraint is imposed in NeuNcut, it will push an entry in each row of $Y$ to $1$ and all others to $0$, which is numerically very difficult due to softmax function.
Note that each row of $Y$ in \eqref{eq:forward} being binary is equivalent to requiring some entries of $\mathbf{g}(\cdot; \Theta)$ to tend to infinite, which poses numerical difficulty in training MLP.
Furthermore, there is no need to yield binary clustering membership in our NeuNcut since that it is convenient to assign the cluster index for each data point via $\texttt{argmax}$.

In addition, the result of the theorem also suggests a practical way to set the penalty parameter $\gamma$. That is, the penalty term $\L_{orth}$ with a reasonable small $\gamma$ in Eq.~\eqref{eq:our_loss} will makes MLP easier to be optimized, by encouraging a soft clustering membership $Y$ but not damaging the correctness of the clustering. 
As we will see in Section~\ref{sec:ablation} that, our NeuNcut can prevent degenerated solutions and obtain satisfactory clustering results when the orthogonality constraint is relaxed---as long as the penalty weight $\gamma$ is larger than a certain threshold.

\subsection{Methods to Learn Affinity}
\label{sec:aff}

We train our NeuNcut with three types of affinity: a) {heat kernel affinity}; b) SiameseNet based heat kernel affinity; and c) self-expressiveness induced affinity.

\myparagraph{Heat kernel affinity}
To be comparable with other classic methods, we implement the most common setting to train our NeuNcut.
The affinity is defined by heat kernel with a bandwidth parameter $\sigma >0$:

\begin{equation}
\label{eq:kernel}
    a_{i,j}=\exp \left(-\frac{\|\x_i-\x_j\|^2_2}{2\sigma^2}\right).
\end{equation}

\myparagraph{SiameseNet based heat kernel affinity}
In SpectralNet \cite{Shaham:ICLR18}, the pairwise affinity is learned by a siamese network, which is trained via a constrastive loss:
\begin{equation}
\label{eq:siamese}
    \mathcal L_{const} (\x_i,\x_j; \Psi)=
    \begin{cases}
    \|\z_i-\z_j\|_2^2,& \x_j \in \mathcal{N}_3(\x_i)\\
    \max (1-\|\z_i-\z_j\|_2,0)^2,& \x_j \in \overline {\mathcal{N}_3}(\x_i), 
    \end{cases}
\end{equation}
where $\Psi$ denotes the parameters of Siamese network, $\z_i$ is the output of Siamese network corresponding to the input $\x_i$, $\mathcal{N}_3(\x_i)$ denotes the set of three nearest neighbors of $\x_i$ determined by Euclidean distance and $\overline{\mathcal{N}_3}(\x_i)$ denotes the set of the three non-neighbors of $\x_i$ which are randomly chosen from all other non-neighbors. 
Once the Siamese network is trained, then the pairwise affinity is defined by 
{the SiameseNet based heat kernel 
as follows:}
\begin{equation}
\label{eq:siamese_distance}
    a_{i,j}=\exp \left(-\frac{\|\z_i-\z_j\|^2_2}{2\sigma^2}\right).
\end{equation}

\myparagraph{Self-expressiveness induced affinity}
When clustering high dimensional data, it is reasonable to assume that data approximately lie on a union of subspaces~\cite{Elhamifar:CVPR09}. 
Here, we adopt SENet~\cite{Zhang:CVPR21-SENet} to learn the self-expressiveness induced affinity, which parametermizes the self-expressive coefficients $c_{ij}$ with a key-query network and train it by minimizing a self-expression loss with elastic net regularization~\cite{You:CVPR16-EnSC}:
\begin{equation}
\label{eq:senet}
        \underset{\{c_{ij}\}_{i\neq j}}{{\min}} \  \frac{\eta}{2}\|{\bm x}_j - \sum_{i\neq j} c_{ij} {\bm x}_i\|_2^2 + \sum_{i\neq j} \left (\lambda |c_{ij}| + \frac{1-\lambda}{2}c_{ij}^2 \right ),
\end{equation}
where $\eta >0$ and $0 \le \lambda \le 1$ are two hyper-parameters.
Given the self-expressive coefficients $c_{ij}$, in default we define the affinity by $a_{ij}=(|c_{ij}|+|c_{ij}|)/2$. 

\section{Experiments}
\label{sec:experiments}

In this section, we provide comprehensive evaluations for our NeuNcut on both synthetic data and real-world data.\footnote{The code of this work will be released upon the acceptance of the manuscript.}

\subsection{Datasets and Metrics}
\label{sec:dataset}

\myparagraph{Datasets} To evaluate the performance of NeuNcut, we use the following datasets:
{\bf MNIST} \cite{Lecun:pe1998} consists of 70,000 samples with gray images of handwritten digits 0-9.
{\bf Fashion MNIST} (F-MNIST) \cite{Xiao:FashionMNIST19} consists of 70,000 samples with 
gray images of 10 fashion products. 
Several fashion product clusters in F-MNIST are hard to distinguish.
{\bf Extended MNIST} (E-MNIST) \cite{Cohen:IJCNN2017-Emnist}, we select all lower case letters with 190,998 images belonging to 26 \emph{extremely imbalanced} categories by following \cite{You:ECCV18}.
{\bf CIFAR-10} \cite{Krizhevsky:CIFAR2009} consists of 60,000 color images belonging to 10 categories. {\bf MNIST8M} \cite{Loosli:2007-Infi}, a.k.a, infinite MNIST, consists of 8 million samples produced from MNIST by using pseudo-random deformations and translations.
{\mcb
{\bf TinyImageNet} and {\bf ImageNet-1k} are two popular subsets of ImageNet~\cite{Deng:CVPR2009-imagenet} that consist color images belonging to 200 and 1,000 categories, respectively.
}

For each dataset, we train our NeuNcut on sampled subset, which is called train dataset, and then evaluate the clustering performance which is yielded by directly inferring the pseudo label via our trained NeuNcut over the \textit{entire} dataset. The size of sampled training subset for each dataset is shown in Table \ref{tab:parameters}.
For MNIST, F-MNIST, E-MNIST and MNIST8M, we compute a feature vector of dimension 3,472 using ScatNet \cite{Bruna2013} and reduce the dimension to 500 using PCA.\footnote{Following \cite{You:ECCV18}, we remove mean after dimension reduction via PCA for E-MNIST.} For CIFAR-10, we use MCR$^2$ \cite{Yu:NIPS20} to extract 128 dimensional features.
{
For TinyImageNet and ImageNet-1k, we use the image encoder of CLIP~\cite{Radford:ICML2021-CLIP}, a large-scale pretrained model, to extract 768 dimensional features, denoted by ``TinyImageNet-CLIP'' and ``ImageNet1k-CLIP'', respectively.
}

\myparagraph{Metrics} We use three common evaluation metrics:
a) clustering accuracy (ACC); b) normalized mutual information (NMI); and c) adjusted rand index (ARI). 
The details of the definitions can be found in the appendices of \cite{Yu:NIPS20}.
In short, the three metrics are ranged in $[0,1]$ and the higher value indicates better performance.

\myparagraph{Parameter Settings}
In NeuNcut, we form a multi-layer preceptrons (MLP) with two hidden layers with ReLU as the activation function. The number of hidden units in each layer is 512.
We use the Adam optimizer with an initial learning rate and cosine annealing learning rate in training.
When using the Euclidean distance-based affinity, we set $\sigma$ in Eq.~\eqref{eq:kernel} to 3 for all datasets. 
Regarding the SiameseNet based heat kernel affinity and the self-expressiveness induced affinity, the parameter setting of Siamese network and SENet are followed by \cite{Shaham:ICLR18} and \cite{Zhang:CVPR21-SENet}, respectively.
Other hyper-parameters of NeuNcut are shown in Table \ref{tab:parameters}.

\begin{table}

\fontsize{9pt}{10pt}\selectfont
\begin{subtable}{0.65\linewidth}
    \centering
    \begin{tabular}{l|ccccccc}
    \toprule
    Dataset & \# Train data & $\gamma$ & $lr$ & $wd$ & $m$ & epochs \\
    \midrule
    MNIST & 20,000 & 100 & 0.005 & $10^{-4}$ & $1000$ & 100\\
    F-MNIST & 50,000 & 100 & 0.005 & $10^{-4}$ & $1000$ & 300\\
    E-MNIST & 50,000 & 350 & 0.01 & 0 & $1000$ & 300\\
    CIFAR-10 & 20,000 & 100 & 0.005 & $10^{-4}$ & $1000$ & 300\\
    MNIST8M & 20,000 & 100 & 0.005 & $10^{-4}$ & $1000$ & 100 \\
    {\mcb TinyImageNet } & 100,000 & 150 & 0.001 & $10^{-4}$ & $1000$ & 100 \\
    {\mcb ImageNet-1k } & 1,281,167 & 500 & 0.001 & $10^{-4}$ & $3000$ & 20  \\
    \bottomrule
    \end{tabular}
    \caption{Best hyper-parameters of NeuNcut on real-world datasets.}
\end{subtable}
\hfill
\begin{subtable}{0.25\linewidth}
\centering
\setlength{\tabcolsep}{0.2mm}{
    \begin{tabular}{l c}
            \toprule
            & Linear: $\RR^{d}\rightarrow \RR^{512}$ \\
            \midrule
            & ReLU \\
            \midrule
            \multirow{2}{*}{2$\times$} & Linear: $\RR^{512}\rightarrow \RR^{512}$ \\
            &  ReLU \\
            \midrule
            & Linear: $\RR^{512}\rightarrow \RR^{k}$ \\
            \midrule
            & Softmax \\
            \bottomrule
    \end{tabular}
    \caption{\mcb Model parameters.}
}
\end{subtable}

\caption{\mcb Hyper-parameters and model parameters of NeuNcut. $\gamma$: penalty weight. $lr$: initial learning rate, $wd$: weight decay, $m$: size of mini-batches, epochs: training epochs, $d$: dimension of data points, $k$: number of clusters.}
\label{tab:parameters}

\end{table}

\begin{table}[tbh]

\fontsize{10pt}{11pt}\selectfont
\centering
    \begin{tabular}{c|l|lll|lll}
    \toprule
    \multirow{2}{*}{Settings} & \multirow{2}{*}{Methods} &  \multicolumn{3}{c|}{MNIST} & \multicolumn{3}{c}{F-MNIST} \\
    & & ACC & NMI & ARI & ACC & NMI & ARI\\
    \midrule
    \multirow{3}{*}{\uppercase\expandafter{\romannumeral1}} & Ncut & 0.693 & 0.809 & 0.665 & 0.557 & 0.629 & 0.437 \\
    & SpectralNet & 0.622* & 0.687* & - & 0.516 & 0.587 & 0.398 \\
    & NeuNcut & 0.784 & 0.817 & 0.734 & 0.602 & 0.603 & 0.473\\
    \midrule
    \multirow{3}{*}{\uppercase\expandafter{\romannumeral2}} & Ncut & 0.821 & 0.867 & 0.795 & 0.561 & 0.645& 0.456\\
    & SpectralNet & 0.800* & 0.814* & - & 0.590 & 0.665 & 0.457 \\
    & NeuNcut & 0.943 & 0.866 & 0.879 & 0.643 & 0.644 & 0.489\\
    \midrule
    \multirow{3}{*}{\uppercase\expandafter{\romannumeral3}} & Ncut & 0.861 & 0.872 & 0.839 & 0.582 & 0.593 & 0.454\\
    & SpectralNet & 0.971 & 0.924 & 0.934 & 0.601 & 0.650 & 0.458\\
    & NeuNcut & 0.969 & 0.922 & 0.914 & 0.650 & 0.659 & 0.501\\
    \midrule
    \multirow{3}{*}{\uppercase\expandafter{\romannumeral4}} &  Ncut & 0.854 & 0.904 & 0.837 & 0.645 & \textbf{0.726} & 0.448\\
    & SpectralNet & 0.816 & 0.833 & 0.752 & 0.617 & 0.645 & 0.438 \\
    & NeuNcut & \underline{0.982} & \underline{0.947} & \underline{0.958} & 0.687 & 0.663 & 0.535\\
    \midrule
    \multirow{3}{*}{\uppercase\expandafter{\romannumeral5}} & Ncut & 0.827 & 0.874 & 0.810 & 0.641 & 0.650 & 0.479\\
    & SpectralNet & 0.841 & 0.901 & 0.834 & \underline{0.694} & 0.670 & \underline{0.544} \\
    & NeuNcut & \textbf{0.983} & \textbf{0.952} & \textbf{0.960} & \textbf{0.783} & \underline{0.725} & \textbf{0.643}\\
    \bottomrule
    \end{tabular}
\caption{Clustering performance of Normalized cut (Ncut), SpectralNet and NeuNcut when combined with different feature and affinity components. (*) are results cited from \cite{Shaham:ICLR18}.}
\label{tab:component}

\end{table}

\subsection{Comparison to Normalized cut \cite{Shi:IEEE2000-Ncut} and SpectralNet \cite{Shaham:ICLR18}}
\label{sec:components}

The feature extraction and affinity learning are important for the success of spectral clustering.
To fairly evaluate the performance of NeuNcut, we conduct experiments on MNIST and F-MNIST under different combinations of feature extraction and affinity learning methods, compared to spectral clustering via Normalized cut \cite{Shi:IEEE2000-Ncut} and SpectralNet \cite{Shaham:ICLR18}. 
We report the experiment results under the following five settings: 
{
(I) Original feature space + heat kernel based affinity; 
(II) Auto-encoders based feature \cite{Jiang:IJCAI2016-VaDE} + heat kernel affinity; 
(III) Auto-encoders based feature + SiameseNet based heat kernel affinity; 
(IV) ScatNet \cite{Bruna2013} based feature + heat kernel affinity; 
}
(V) ScatNet based feature + self-expressiveness induced affinity.
Experimental results are provided in Table \ref{tab:component}.
We can read that NeuNcut outperforms Normalized cut \cite{Shi:IEEE2000-Ncut} and SpectralNet \cite{Shaham:ICLR18} in almost all settings. NeuNcut achieves the best performance when combined with ScatNet and the self-expressiveness induced affinity. 
In the following, we use ScatNet based feature + self-expressiveness induced affinity (\ie, the setting V in Table~\ref{tab:component}) as the default setting for MNIST and F-MNIST.

\subsection{Evaluation on Generalization Performance}
\label{sec:generalization}

To evaluate the generalization performance, we randomly select a number of samples for training NeuNcut and evaluate the directly inferred clustering results on the entire dataset. 
We also report the performance of Normalized cut (Ncut) \cite{Shi:IEEE2000-Ncut} and SpectralNet \cite{Shaham:ICLR18} on entire datasets.
Conventional spectral clustering has no generalization ability; whereas SpectralNet learns an embedding from graph Laplacian but has to perform $k$-means on the entire dataset.
All methods use ScatNet \cite{Bruna2013} based feature and self-expressiveness induced affinity.
As can be seen in Table \ref{tab:generalization}, our NeuNcut trained by a small subset with more than 10,000 samples outperforms Ncut \cite{Shi:IEEE2000-Ncut} and SpectralNet \cite{Shaham:ICLR18}, which validates the good generalization ability of NeuNcut for out-of-sample-extensions. 

\begin{table}

\fontsize{10pt}{11pt}\selectfont
\centering
    \begin{tabular}{l|c|ccc|ccc}
        \toprule
        \multirow{2}{*}{Methods} & \multirow{2}{*}{Train data} & \multicolumn{3}{c|}{MNIST} & \multicolumn{3}{c}{F-MNIST} \\
        & & ACC & NMI & ARI & ACC & NMI & ARI \\
        \midrule
        Ncut & NA & 0.827 & 0.874 & 0.810 & 0.641 & 0.650 & 0.479\\
        SpectralNet & NA & 0.841 & 0.901 & 0.834 & 0.694 & 0.670 & 0.544\\
        \midrule
        \multirow{6}{*}{NeuNcut}
        & 1,000 & 0.711 & 0.699 & 0.616 & 0.572 & 0.510 & 0.389 \\
        & 2,000 & 0.858 & 0.783 & 0.749 & 0.659 & 0.616 & 0.487 \\
        & 5,000 & 0.967 & 0.913 & 0.928 & 0.694 & 0.638 & 0.537 \\
        & 10,000 & 0.977 & 0.938 & 0.953 & 0.717 & 0.692 & 0.598 \\
        & 20,000 & \textbf{0.983} & \textbf{0.952} & \textbf{0.960} & 0.752 & 0.695 & 0.618\\
        & 50,000 & 0.981 & 0.946 & 0.957 & \textbf{0.783} & \textbf{0.725} & \textbf{0.643}\\
        \bottomrule
    \end{tabular}
    \caption{Clustering performance of Ncut, SpectralNet and NeuNcut with varying number of training samples on MNIST and F-MNIST datasets.}
    \label{tab:generalization}
\end{table}

\subsection{Experiments on Large-scale Data}
For normalized cut \cite{Shi:IEEE2000-Ncut}, there is a computation bottleneck when computing the embedding of the entire dataset if the size of the data is too large. By contrast, our NeuNcut takes only a small affinity graph of each mini-batch data and thus enjoys a good scalability to handle large-scale data.

\myparagraph{On large-scale synthetic data}
We generate some typical 2D synthetic data and visualize the predictions of NeuNcut. ``Double rings" contains 2 concentric circles and ``double C" contains two clusters in the shape of letter ``C". Each dataset contains 10 million data points.
We use $10,000$ samples (\ie, only 0.1\% samples are used) for training and construct heat kernel based affinities. As can be observed from Figure \ref{fig:visual} that, our NeuNcut yields nearly 100\% correct clusters.

\begin{figure}
\small
\centering
\begin{subfigure}{0.4\linewidth}
    \centering
    \includegraphics[totalheight=0.6\linewidth,width=0.6\linewidth]{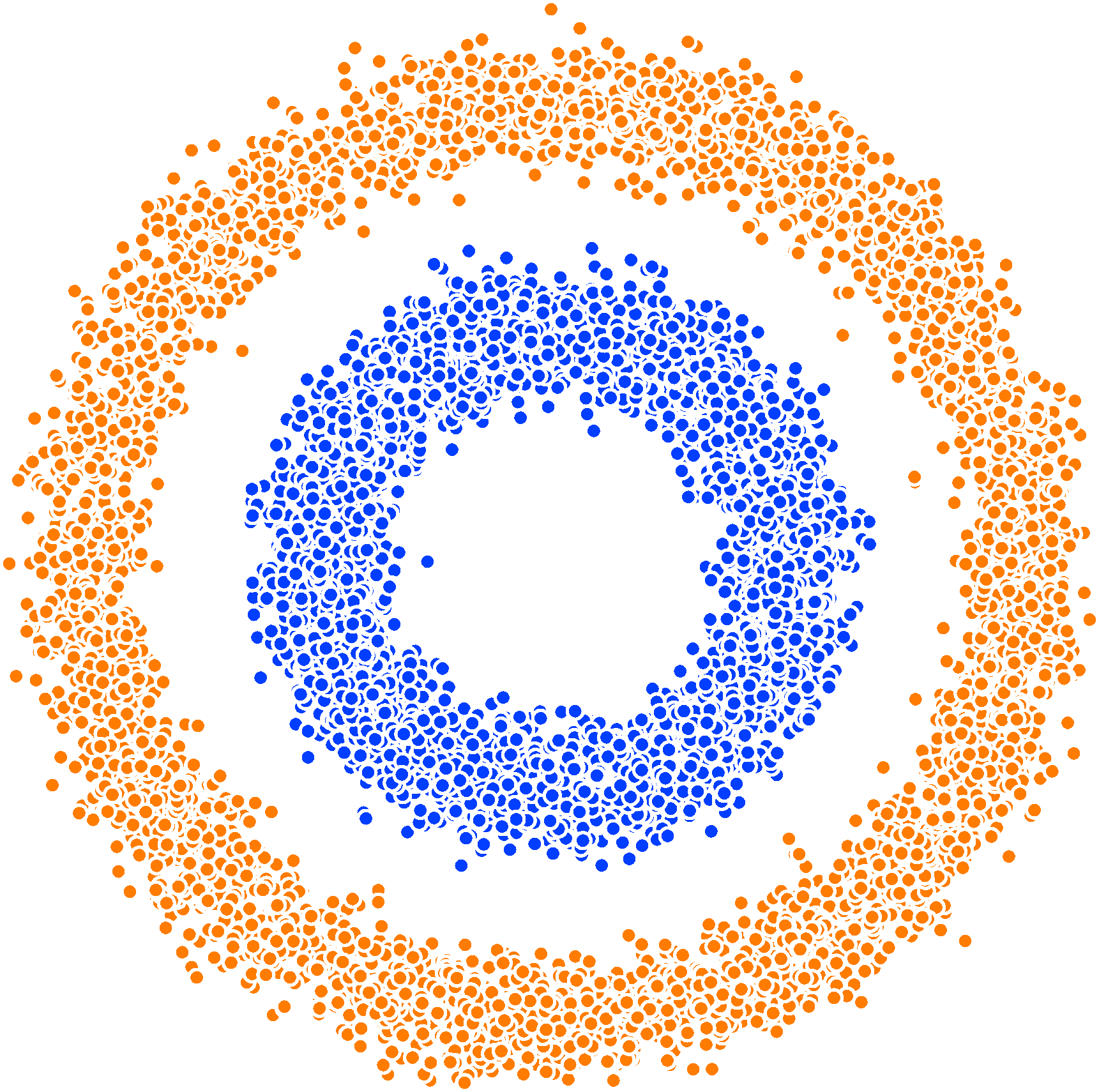}
    \vspace{8pt}
    \caption{Double rings (ACC=100.00\%)}
\end{subfigure}
\begin{subfigure}{0.4\linewidth}
    \centering
    \includegraphics[totalheight=0.6\linewidth,width=0.6\linewidth]{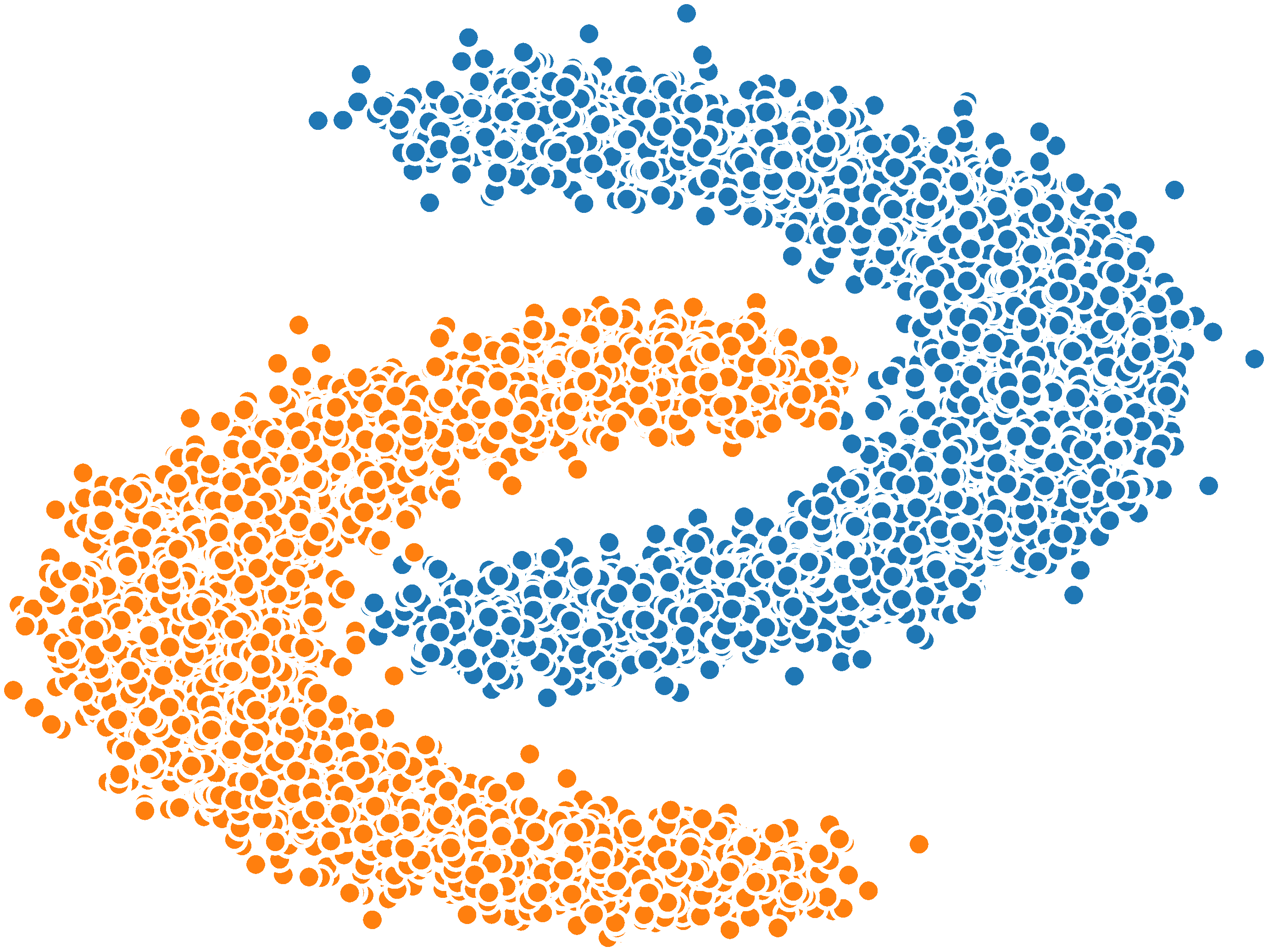}
    \vspace{8pt}
    \caption{Double C (ACC=99.90\%).}
\end{subfigure}
\caption{Visualization of NeuNcut predictions on large-scale synthetic data. 10,000 samples are plotted.}
\label{fig:visual}
\end{figure}

\myparagraph{On large-scale real data}
We compare the performance and total running time  of spectral clustering via Normalized cut (Ncut), SpectralNet~\cite{Shaham:ICLR18} and NeuNcut on MNIST8M, TinyImageNet-CLIP and ImageNet1k-CLIP. 
For MNIST8M, we construct affinities from a pretrained SENet \cite{Zhang:CVPR21-SENet}. 
For  TinyImageNet-CLIP and ImageNet1k-CLIP, we compute the heat kernel affinities.
Due to the computational bottleneck of eigen-decomposition, we report the average performance of Ncut over all subsets that contain $100,000$ samples.
Table \ref{tab:large_dataset} shows that NeuNcut achieves satisfactory clustering accuracy on three large-scale real datasets. Besides, NeuNcut does not need to perform $k$-means on the datasets and thus saves more time of the inference when compared to SpectralNet. 

\begin{table}

\centering
\fontsize{10pt}{11pt}\selectfont
\setlength{\tabcolsep}{0.75mm}{
\centering
    \begin{tabular}{l|cccc|cccc|cccc}
    \toprule
    \multirow{2}{*}{Methods} & \multicolumn{4}{c|}{MNIST8M} & \multicolumn{4}{c|}{{\mcb TinyImageNet-CLIP}} & \multicolumn{4}{c}{\mcb ImageNet1k-CLIP} \\
    & Time & ACC & NMI & ARI & Time & ACC & NMI & ARI & Time & ACC & NMI & ARI \\
    \midrule
    Ncut & 3.22$\times 10^3$ & 0.960 & 0.915 & 0.918 & 51.47 & 0.628 & 0.770 & 0.447 & 533.55 & 0.560 & 0.812 & 0.425 \\
    SpectralNet & 2.27 & 0.961 & 0.913 & 0.920 & 9.97 & 0.622 & 0.766 & 0.419 & 19.22 & 0.560 & 0.798 & 0.441 \\
    NeuNcut & \bf 1.36 & \bf 0.968 & \bf 0.923 & \bf 0.926 & \bf 5.40 & \bf 0.646 & \bf 0.781 & \bf 0.506 & \bf 5.23 & \bf 0.624 & \bf 0.813 & \bf 0.473 \\
    \bottomrule
    \end{tabular}
    }
\caption{\mcb Total running time (min.), ACC, NMI and ARI of Ncut, SpectralNet, and NeuNcut on MNIST8M, TinyImageNet-CLIP and ImageNet1k-CLIP.}
\label{tab:large_dataset}

\end{table}

\subsection{Ablation Studies}
\label{sec:ablation}
In this section, we provide ablation studies on the network size, size of mini-batch data and hyper-parameter $\gamma$.
We train NeuNcut on randomly selected 20,000 samples under the setting V (\ie, ScatNet based feature + self-expressiveness induced affinity) and then directly infer the cluster memberships.

\myparagraph{Effect of Network Size}
To evaluate how the network size affects the clustering performance, we conduct experiments on MNIST with varying the number of hidden layers (\ie, the layer depth) used for NeuNcut in the range of $\{1,2,3\}$ and varying the number of neurons in each hidden layer (\ie, the hidden dimension) in the range of $\{128,256,512,1024\}$. In these experiments, we set the batch size to 1000. 
Experimental results are provided in Figure \ref{fig:MNIST_model_design}, where we display the performance of ACC and NMI under different network size.
We can observe that the neural network with 
two hidden layers and each layer containing 
512 neurons can achieve the best clustering performance. Further increasing the 
depth does not help to improve the clustering performance, but requires more training time.
\begin{figure}[tbh]
\centering
    \begin{subfigure}{0.45\linewidth}
        \centering
        \includegraphics[width=0.85\linewidth]{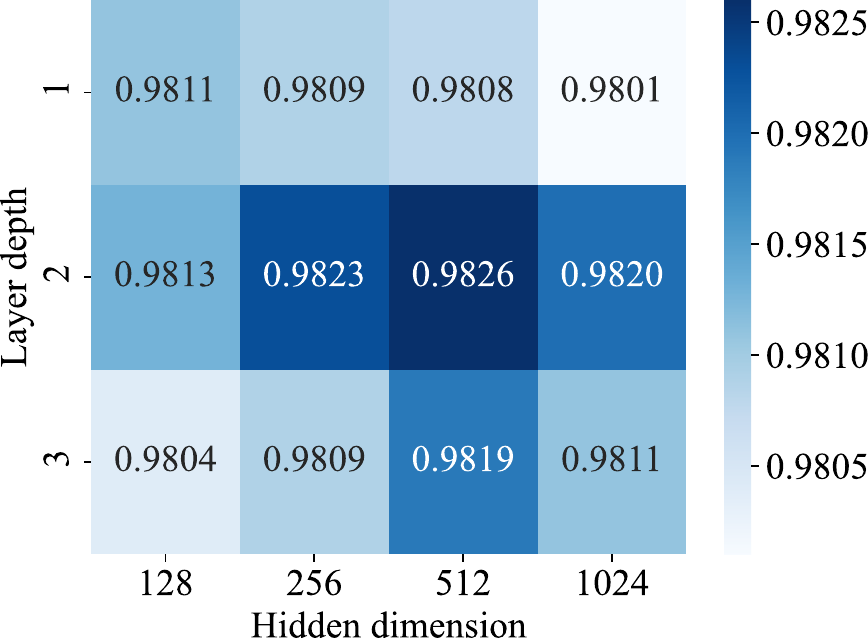}
        \vspace{8pt}
        \caption{ACC}
    \end{subfigure}
    \begin{subfigure}{0.45\linewidth}
        \centering
        \includegraphics[width=0.85\linewidth]{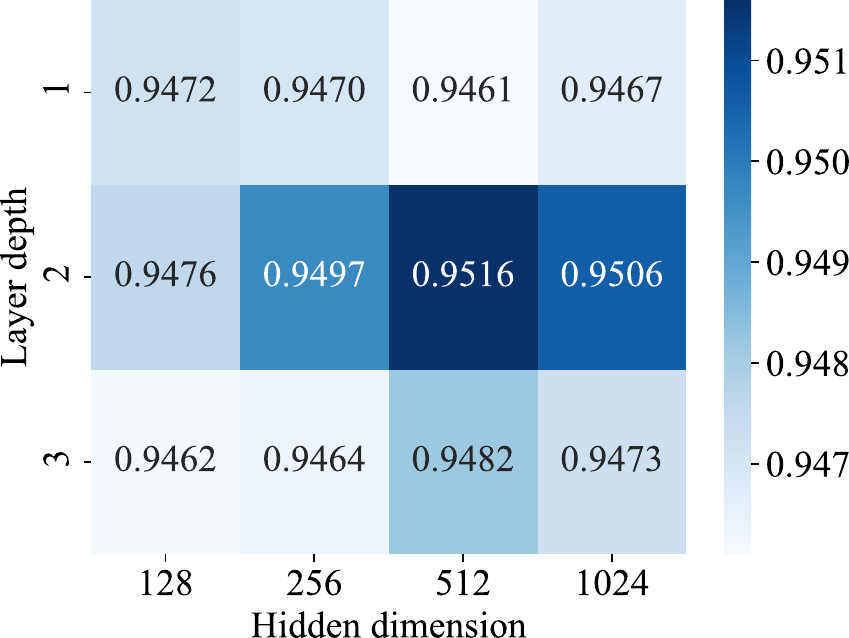}
        \vspace{8pt}
        \caption{NMI}
    \end{subfigure}

    \caption{Clustering performance of NeuNcut with varying network size on MNIST.}
    \label{fig:MNIST_model_design}
\end{figure}

\myparagraph{Effect of Batch Size}
{
The batch size is important for approximating the graph Laplacian $L$ with its mini-batch version.
To evaluate the impact of batch size, we train NeuNcut with varying batch size in the range of $\{50,100,200,\dots,1400\}$ and report the mean value of ACC over 3 trials.
As can be seen in Table \ref{tab:bs}, NeuNcut yields satisfactory results whenever the batch size is larger than 100 on MNIST.
For E-MNIST that contains more imbalanced categories, NeuNcut demands a batch size larger than 200 to yield satisfactory results.
}

\begin{table}

    \footnotesize
    \centering
    \begin{tabular}{c|c c c c c c c c c}
        \toprule
         Batch size & 50 & 100 & 200 & 400 & 600 & 800 & 1000 & 1200 & 1400\\
        \midrule
         MNIST & 0.751 & 0.973 & 0.978 & 0.980 & 0.982 & 0.980 & 0.983 & 0.977 & 0.981 \\
         {E-MNIST} & 0.463 &  0.477 &  0.684 &  0.709 &  0.711 &  0.708 &  0.716 &  0.709 & 0.713 \\
         \bottomrule
    \end{tabular}
\caption{Clustering accuracy of NeuNcut with varying batch size on MNIST and E-MNIST.}
\label{tab:bs}

\end{table}

\myparagraph{Effect of Hyper-parameter $\gamma$}
Note that Eq.~\eqref{eq:our_loss} introduce an additional hyper-parameter $\gamma >0$. We conduct experiments with NeuNcut under different $\gamma$ and show the results in Figure \ref{fig:weight}.
The clustering performance of our method is not sensitive to the parameter $\gamma$, \eg, NeuNcut achieves stable results on MNIST when $\gamma \in [80,240]$ and $\gamma \in [80,180]$ when it comes to F-MNIST.

\begin{figure}
\centering
\begin{subfigure}{0.4\linewidth}
    \centering
    \includegraphics[width=0.9\linewidth]{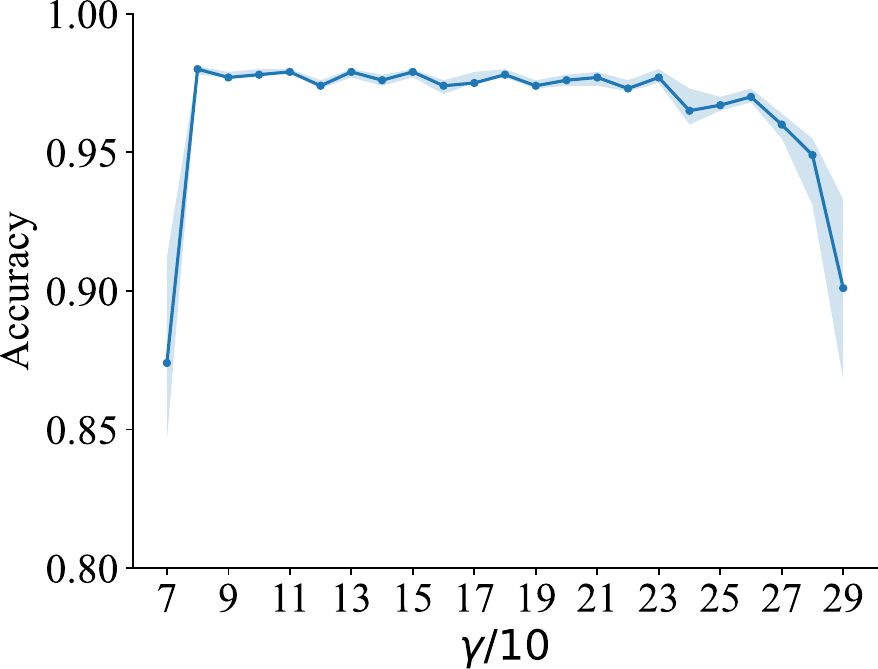}
    \caption{MNIST}
    \label{subfig:MNIST-weight}
\end{subfigure}
\begin{subfigure}{0.4\linewidth}
    \centering
    \includegraphics[width=0.9\linewidth]{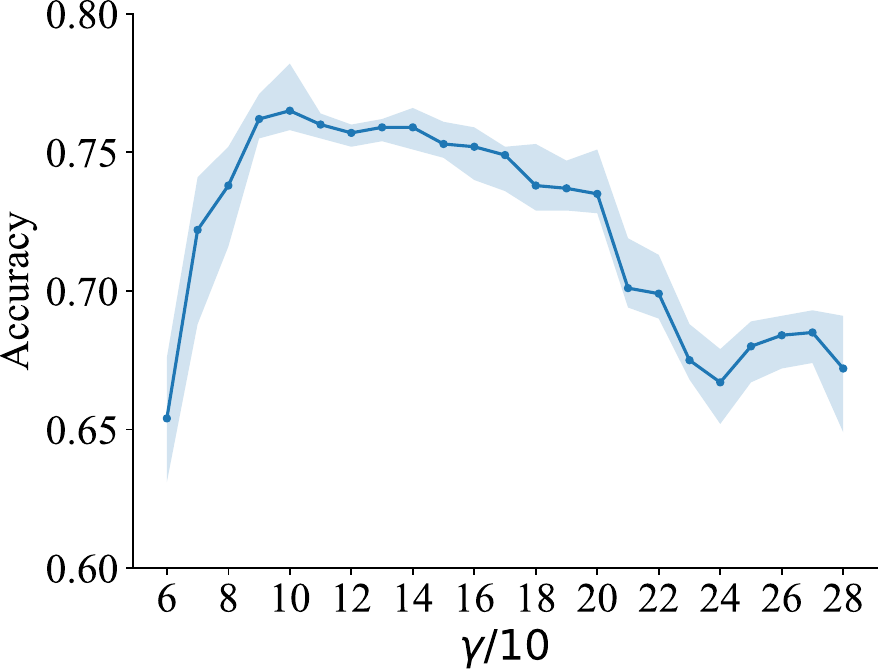}
    \caption{F-MNIST}
\end{subfigure}
    \caption{Clustering accuracy (mean$\pm$std) of NeuNcut with varying $\gamma$ on MNIST and F-MNIST.}
    \label{fig:weight}
\end{figure}

\subsection{Searching Best $\gamma$ Without Labels}
\label{sec:gamma_search}

Since label information is not available in training, it is improper to search for the best hyper-parameters by checking the clustering accuracy. To this end, we provide a practical way to find the best $\gamma$ without using the ground-truth labels.
Note that the loss function in Eq.~\eqref{eq:our_loss} consists of two terms, \ie, $\mathcal L_{Lap}$ and $\mathcal L_{orth}$. 
We observe some fascinating connections between the optimal loss and the parameter $\gamma$.
Taking the experiment on {MNIST} as an example, we train our NeuNcut with different $\gamma$ and record the optimal $\mathcal L_{Lap}$ and $\mathcal L_{orth}$ during the entire training process.
As can be seen in Figure~\ref{fig:gamma_threshold}, when $\gamma$ is small, $\mathcal L_{orth}$ is larger than its recorded minimum, while $\mathcal L_{Lap}$ is equal to $0$ or very small. In this case our NeuNcut will produce a collapsed solution.
When $\gamma$ reaches a threshold, $\mathcal L_{orth}$ approaches to its minimum and in this case NeuNcut can produce satisfactory clustering results.
Further increasing the $\gamma$ slightly increases the value of $\mathcal L_{Lap}$ while the value of $\mathcal L_{Lap}$ is unchanged.
Finally, an arbitrary large $\gamma$ will makes $\mathcal L_{Lap}$ hard to be optimized and harms the clustering accuracy.

Based on these observations, we suggest an empirical rule to find the best $\gamma$. 
Specifically, we start from a very large $\gamma$ (\eg, $\gamma=10^6$) to train NeuNcut and record the optimal $\mathcal L^{(o)}_{orth}$ (\ie, where $\mathcal L^{(o)}_{orth}=0.256$ in this case)---which can be regarded as the lower bound of $\mathcal L_{orth}$.
Then, we gradually decrease the value of $\gamma$ and find the threshold of $\gamma$ that corresponds to the lowest one among the optimal $\mathcal L_{Lap}$ while still keeping $\mathcal L_{orth}=0.256$.
This estimated threshold sets the lower bound of feasible $\gamma$ and is consistent with the real feasible $\gamma$, \ie, $\gamma\geq 80$ as can be validated by Figure~\ref{subfig:MNIST-weight}.

\begin{figure}
\centering
\begin{subfigure}{0.4\linewidth}
    \centering
    \includegraphics[width=0.9\linewidth]{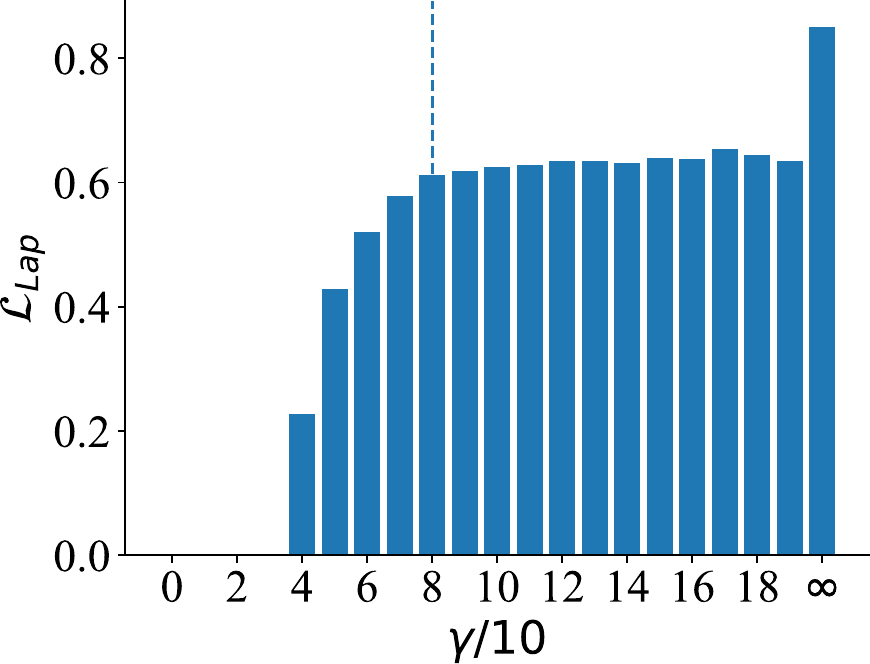}
    \caption{$\mathcal L_{Lap}$}
\end{subfigure}
\begin{subfigure}{0.4\linewidth}
        \centering
        \includegraphics[width=0.9\linewidth]{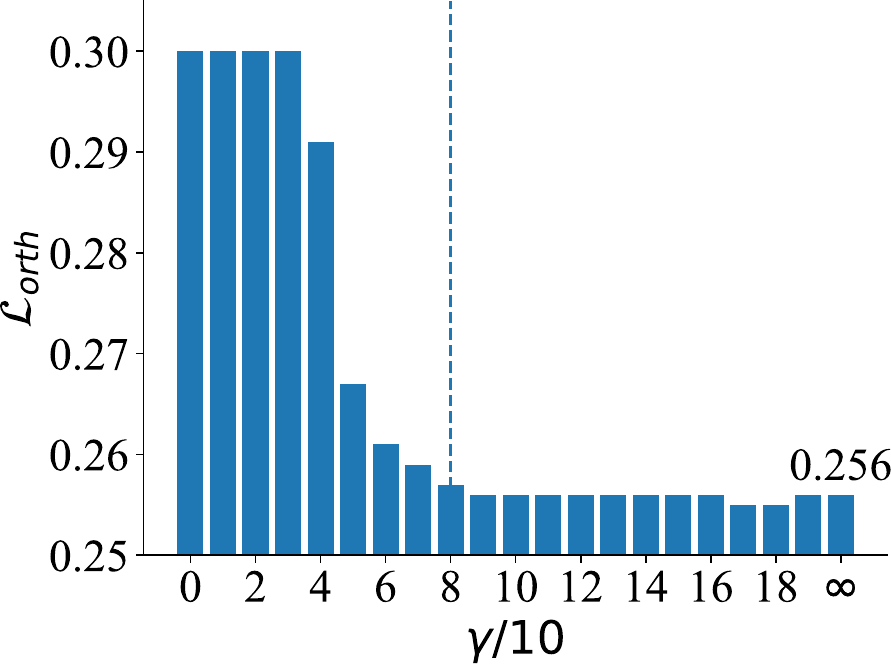}
        \caption{$\mathcal L_{orth}$}
\end{subfigure}
\caption{Showing the optimal $\mathcal L_{Lap}$ in panel (a) and the optimal $\mathcal L_{orth}$ in panel (b) with a varying penalty weight $\gamma$ in \eqref{eq:our_loss} when training on MNIST.}
\label{fig:gamma_threshold}
\end{figure}

\subsection{Computation Cost and Running Time}

\begin{table}[tbh]

\fontsize{10pt}{11pt}\selectfont
    \centering
    \begin{tabular}{l|c|c|c|c}
        \toprule
        Methods & MNIST & F-MNIST & E-MNIST & CIFAR-10\\
        \midrule
        Ncut & 283 & 334 & 869 & 312 \\
        \midrule
        NeuNcut (5000) & 11 & 47 & 43 & 21\\
        NeuNcut (10000) & 13 & 58 & 51 & 40\\
        NeuNcut (20000) & 25 & 92 & 97 & 101 \\
        NeuNcut (50000) & 68 & 258 & 249 & 211 \\
        \bottomrule
    \end{tabular}
    \caption{Total running time~(sec.) of Ncut and NeuNcut ($N$) on MNIST, F-MNIST, E-MNIST and CIFAR-10. Ncut is performed on the entire dataset and $N$ denotes the number of training samples. }
    \label{tab:time}
\end{table}

\myparagraph{Computation Complexity}
Considering a dataset containing $n$ data points, the common strategy for spectral clustering is to sparsify the affinity at first, \eg, keeping the largest $s$ entries in each row.
In such a setting, the time complexity is $\mathcal O (kns)$ for solving the $k$ eigenvectors with sparse eigen-solver.
While in NeuNcut, the time complexity of the loss in Eq.~\eqref{eq:our_loss} is $\mathcal O (2tkm^2)$, where $t$ denotes the number of training iterations and $m$ denotes the batch size.

\myparagraph{Running Time}
We compare the total running time of NeuNcut with varying number of training samples and the running time of spectral clustering via Ncut on the entire dataset MNIST, F-MNIST, E-MNIST and CIFAR-10.
We use an Intel(R) Xeon E5-2630 CPU to solve spectral clustering via Ncut since that there is no available GPU acceleration package.
The NeuNcut is trained on a single NVIDIA GeForce 1080Ti GPU.
As shown in Table \ref{tab:time}, the NeuNcut takes much less running time since it can be trained on a small training set to perform generalizable clustering results.
Besides, NeuNcut infers the cluster memberships directly, which saves the time of applying $k$-means clustering. 

\subsection{Learning Curves}

\begin{figure}[t]
    \centering
    \begin{subfigure}{0.4\linewidth}
        \centering
        \includegraphics[width=0.8\linewidth, height=0.6\linewidth]{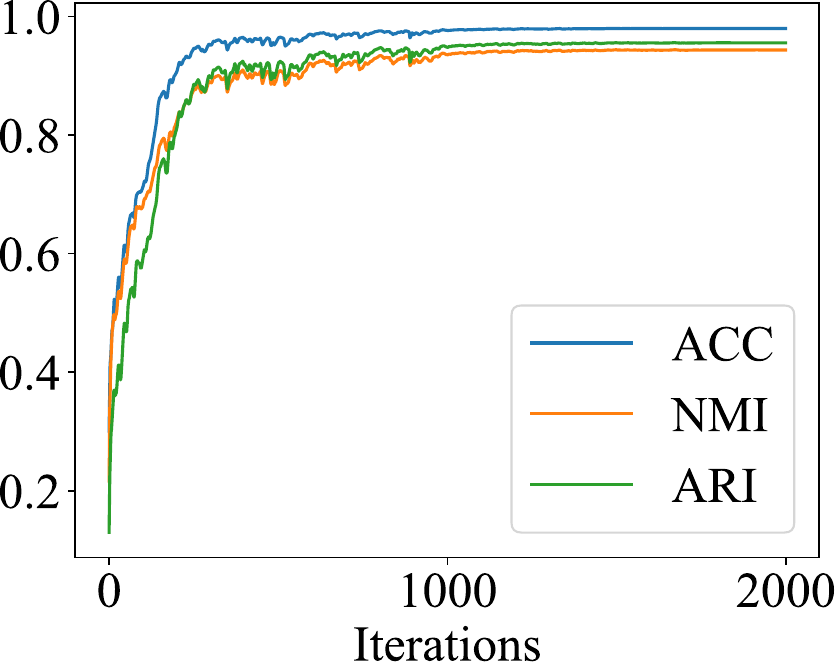}
        \caption{Metrics}
    \end{subfigure}
    \begin{subfigure}{0.4\linewidth}
        \centering
        \includegraphics[width=0.8\linewidth, height=0.6\linewidth]{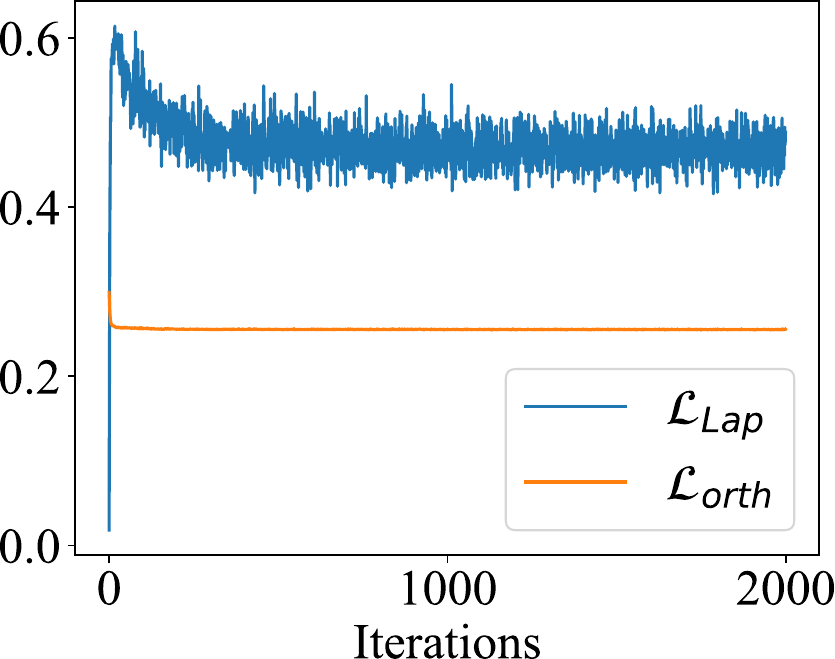}
        \caption{$\L_{Lap}$ and $\L_{orth}$}
    \end{subfigure}
  \caption{Clustering performance curves and loss curves of during training NeuNcut on MNIST.}
  \label{fig:mnist_learning_curves}
\end{figure}

In Figure \ref{fig:mnist_learning_curves}, we plot the clustering performance (ACC, NMI and ARI metrics) curves as well as the loss curves $\L_{Lap}$ and $\L_{orth}$ as defined in Eq.~\eqref{eq:our_loss} during training on MNIST.
We can observe that our NeuNcut converges and obtains the stable clustering results on MNIST within 1,000 training iteration.
The penalty term $\L_{orth}$ decrease steadily to their lower bounds, while the loss term $\L_{Lap}$ rapidly increases from 0 and then slightly decreases during the training.

\begin{table}

\fontsize{10pt}{11pt}\selectfont
\centering
    \begin{tabular}{l|lll|lll}
    \toprule
    \multirow{2}{*}{Methods} &  \multicolumn{3}{c|}{MNIST} & \multicolumn{3}{c}{F-MNIST} \\
     & ACC & NMI & ARI & ACC & NMI & ARI\\
    \midrule
    Ncut ($s$=1000) & 0.798 & 0.812 & 0.727 & 0.641 & 0.650 & 0.480\\
    Ncut ($s$=10) & 0.827 & 0.874 & 0.810 & 0.458 & 0.445 & 0.212\\
    Ncut ($s$=3) & 0.721 & 0.818 & 0.708 & 0.405 & 0.368 & 0.159\\
    \midrule
    SpectralNet ($s$=1000) & 0.795 & 0.815 & 0.744 & 0.667 & 0.639 & 0.524\\
    SpectralNet ($s$=10) & 0.831 & 0.882 & 0.813 & 0.694 & 0.670 & 0.544\\
    SpectralNet ($s$=3) & 0.841 & 0.901 & 0.960 & 0.582 & 0.685 & 0.485\\
    \midrule
    NeuNcut ($s$=1000) & 0.981 & 0.945 & 0.957 & 0.783 & 0.725 & 0.643\\
    NeuNcut ($s$=10) & 0.982 & 0.947 & 0.959 & 0.764 & 0.709 & 0.618\\
    NeuNcut ($s$=3) & 0.983 & 0.952 & 0.960 & 0.758 & 0.701 & 0.606\\
    \bottomrule
    \end{tabular}
\caption{Performance of spectral clustering with Ncut, SpectralNet and NeuNcut on MNIST and F-MNIST with varying $s$, which is the number of nonzero affinity entries kept in each row.}
\label{tab:knn}

\end{table}

\subsection{Evaluation on Sparsity of Affinity}
Because of the memory bottleneck, spectral clustering have to reserve only $s$ largest entries of each row to construct a sparse affinity matrix.
However, setting all other entries to zeros may affect the clustering performance.
Here, we compares the clustering performance of Ncut, SpectralNet and NeuNcut with $s\in \{3,10,1000\}$.
Table \ref{tab:knn} shows that the number of $s$ severely affects the clustering performance of spectral clustering and SpectralNet.
Moreover, since the optimal number of $s$ varies from different dataset, it is nearly impossible to find best $s$ when ground truth is missing or unknown. For example, SpectralNet achieves best clustering results for MNIST when $s=3$ but achieves worst results for F-MNIST with same $s$. 
By contrast, NeuNcut shows robustness to varying $s$.

\subsection{Comparison to State-of-the-art Methods}

\begin{table}[tbh]
\fontsize{9pt}{10pt}\selectfont
\centering
\setlength{\tabcolsep}{1.1mm}{
    \begin{tabular}{l|ccc|ccc|ccc|ccc}
    \toprule
    \multirow{2}{*}{Methods} & \multicolumn{3}{c|}{MNIST} & \multicolumn{3}{c|}{F-MNIST} & \multicolumn{3}{c|}{E-MNIST} & \multicolumn{3}{c}{CIFAR-10} \\
    & ACC & NMI & ARI & ACC & NMI & ARI & ACC & NMI & ARI & ACC & NMI & ARI \\
    \midrule
    $k$-means \cite{MacQueen-1967} & 0.541 & 0.507 & 0.367 & 0.505 & 0.578 & 0.403 & 0.459 & 0.438 & 0.316 & 0.525 & 0.589 & 0.276 \\
    VaDE \cite{Jiang:IJCAI2016-VaDE} & 0.963 & 0.912 & 0.922 & 0.604 & 0.641 & 0.477 & 0.561 & 0.694 & 0.518 & - & - & -  \\
    EnSC \cite{You:CVPR16-EnSC} & \underline{0.980} & \underline{0.945} & \underline{0.957} & 0.672 & 0.705 & 0.565 & T & T & T & 0.613 & 0.601 & 0.430 \\
    DEPICT \cite{Ghasedi:ICCV17-DEPICT} & 0.965 & 0.917 & - & 0.392 & 0.392 & - & - & - & - & - & - & - \\
    ESC \cite{You:ECCV18} & 0.971 & 0.925 & 0.936 & 0.668 & 0.708 & 0.556 & \textbf{0.732} & \textbf{0.825} & 0.759 & 0.653 & 0.629 & 0.438 \\
    SCAN \cite{Van:ECCV2020-SCAN} & 0.969 & 0.916 & 0.929 & 0.538 & 0.575 & 0.363 & 0.567 & 0.652 & 0.545 & 0.756 & 0.633 & \underline{0.577} \\
    SENet \cite{Zhang:CVPR21-SENet} & 0.968 & 0.918 & 0.931 & 0.697 & 0.663 & 0.556 & \underline{0.721} &  \underline{0.798} & \underline{0.766} & \underline{0.765} & \textbf{0.655} & 0.573 \\
    EDESC \cite{Cai:CVPR2022-EDESC} & 0.913 & 0.862 & - & 0.631 & 0.670 & - & - & - & - & 0.627 & 0.464 & - \\
    \midrule
    Ncut \cite{Shi:IEEE2000-Ncut} & 0.854 & 0.904 & 0.837 & 0.645 & \textbf{0.726} & 0.448 & 0.662 & 0.769 & 0.654 & 0.693 & 0.636 & 0.428 \\
    SpectralNet \cite{Shaham:ICLR18} & 0.971 & 0.924 & 0.934 & 0.694 & 0.670 & 0.544 & 0.556 & 0.750 & 0.556 & 0.728 & 0.624 & 0.546 \\
    {\mcb SpecNet2} \cite{Chen:MSML2022-SpecNet2}  & 0.974 & 0.937 & 0.940 & 0.680 & 0.676 & 0.542 & 0.570 & 0.753 & 0.575 & 0.696 & 0.641 & 0.531 \\
    {\mcb AutoSC} \cite{Fan:NIPS2022-AutoSC} & 0.978 & - & - & 0.646 & - & - & - & - & - & - & - & - \\
    CNC \cite{Nazi:arXiv2019-CNC} & 0.972 & 0.924 & - & - & - & - & - & - & - & 0.702 & 0.586 & - \\

    NeuNcut (ours) & \textbf{0.983} & \textbf{0.952} & \textbf{0.960} & \textbf{0.783} & \underline{0.725} & \textbf{0.643} &  0.716 & 0.789 & \textbf{0.774} & \textbf{0.776} & \underline{0.647} & \textbf{0.594} \\
    NeuRcut (ours) & 0.978 & 0.938 & 0.951 & \underline{0.713} & 0.693 & \underline{0.587} & 0.661 & 0.610 & 0.595 & 0.759 & 0.640 & 0.566 \\
    \bottomrule
    \end{tabular}}

\caption{\mcb Clustering results on MNIST, F-MNIST, E-MNIST and CIFAR-10. We compare our method with most relevant spectral clustering methods and other baseline clustering methods. Legend: `-' denotes not reported results, `T' means the computation time exceeds 24 hours.}
\label{tab:sota}
\end{table}

{
We compare the clustering performance of our NeuNcut to the following baselines 
that are most relevant methods to our NeuNcut\footnote{BaSiS learns spectral embeddings in a \textit{supervised} manner and thus it is not being compared.}, including SpectralNet \cite{Shaham:ICLR18}, SpecNet2~\cite{Chen:MSML2022-SpecNet2}, AutoSC~\cite{Fan:NIPS2022-AutoSC} and CNC \cite{Nazi:arXiv2019-CNC},
}
conventional clustering methods including $k$-means \cite{MacQueen-1967} and Normalize cut (Ncut) \cite{Shi:IEEE2000-Ncut},
a set of competitive clustering methods including VaDE \cite{Jiang:IJCAI2016-VaDE}, DEPICT \cite{Ghasedi:ICCV17-DEPICT}, SCAN \cite{Van:ECCV2020-SCAN} and EDESC \cite{Cai:CVPR2022-EDESC},
and a set of advanced subspace clustering methods including EnSC \cite{You:CVPR16-EnSC},  ESC \cite{You:ECCV18} and SENet \cite{Zhang:CVPR21-SENet}.
Among the relevant baselines, we reproduce spectral clustering via Ncut, SpectralNet, SpecNet2 and SCAN and report their best performances on the same features and affinities, and cite the results of AutoSC\footnote{AutoSC uses ScatNet for feature extraction and constructs the affinity via a self-expressive model, thus it is a fair comparison to ours by citing the results from its paper.} 
and CNC\footnote{CNC is a relevant approach to ours, but we failed to reproduce its performance.} 
Experimental results are reported in Table~\ref{tab:sota}.
Among these baselines, DEPICT and SCAN use an entropy based regularizer to avoid degenerated solutions. The success of these methods is owing to class-balanced prior, \ie, assuming all clusters have equal number of data points. These methods perform poorly on imbalanced datasets such as E-MNIST.
State-of-the-art subspace clustering methods including EnSC, ESC, SENet and EDESC generate self-expressive coefficients and apply spectral clustering. Our NeuNcut can be used to replace the spectral clustering step in most subspace clustering methods to achieve better performance and enable them to handle large-scale datasets (\eg, see ``SENet'' and ``NeuNcut'' in Table~\ref{tab:sota} for comparison).
As a differential spectral clustering method, our NeuNcut outperforms Ncut and other spectral-based clustering methods on all four datasets.
In particular, our NeuNcut outperforms all baseline methods on MNIST, F-MNIST and CIFAR-10, and achieves second highest accuracy on E-MNIST.

\subsection{Flexible Extensions}
\label{sec:extension}
Our NeuNcut can be easily extended to other conventional spectral clustering methods, such as Ratio cut \cite{Chan:IDAC1993-Rcut}:
\begin{equation}
\label{eq:rcut}
\mathrm{Ratio\ cut}(\V^{(1)},\ldots,\V^{(k)}) := \sum_{\ell=1}^k \frac{\mathrm{cut}(\V^{(\ell)},\overline{\V}^{(\ell)})}{|\V^{(\ell)}|},
\end{equation}
where $|\V^{(\ell)}|$ denotes the size 
of $\ell$-th cluster. Similarly, the desired segmentation matrix $\tilde H$ of the Ratio cut can again be expressed as $\tilde H := H \Upsilon^{-1}$,  
where $H$ is reparametrized by $\mathbf f(X; \Theta)$ and 
$\Upsilon = \mathrm{Diag} (|\V^{(1)}|,\cdots,|\V^{(k)}|)^{1/2}$. 
Here, the size of $\ell$-th cluster can also be replaced by its estimation $|\V^{(\ell)}| \doteq \sum_{i=1}^n y_{i,\ell}$. 
Again, this is an analogue to the expectation step in EM-style algorithm. 
Then, we relax the orthogonality constraint of Ratio cut (\ie, $\tilde H^\top \tilde H=I$) to a penalty function and obtain the following loss function:
\begin{equation}
\begin{split}
        \mathcal{L}(X,W;\Theta) :=&  \mathrm{\trace} \left ((\mathbf f(X; \Theta) \Upsilon^{-1})^\top L (\mathbf f(X; \Theta) \Upsilon^{-1}) \right )\\
        +& \frac{\gamma}{2} \left \|(\mathbf f(X; \Theta) \Upsilon^{-1})^\top (\mathbf f(X; \Theta) \Upsilon^{-1})-I \right \|_F^2. 
\end{split}
\end{equation}

We note that SpectralNet \cite{Shaham:ICLR18} is not a fully differential programming approach for Ratio cut because $k$-means is still needed after the spectral embedding.
Here, we have a fully differential programming approach for Ratio cut, termed as \textbf{NeuRcut}.
We conduct experiments with the same setting as NeuNcut and report the results in Table \ref{tab:sota}. Again, our NeuRcut outperforms SpectralNet. 

\section{Conclusion}
\label{sec:conclusion}
We proposed a differential and generalizable approach for spectral clustering, termed Neural Normalized Cut (NeuNcut), which can be trained in mini-batch mode and used to infer the clustering membership for out-of-sample data directly. 
Such a generalization ability provides an efficient and effective way for clustering large-scale data. 
Extensive experiments on both synthetic data and real-world datasets have validated the superior performance of our proposed NeuNcut.

\myparagraph{Limitations}
Our NeuNcut maps the data to the cluster assignment space, not the orthogonal eigenfunctions space.
That means, the output of NeuNcut can not approximate the eigenvectors obtained by eigen-decomposition, which prevents potential applications on eigenvectors or spectral embedding, such as the approximation of the Fiedler vector and the positional encoding for graph neural networks.

\myparagraph{Future works}
It is worth to note that NeuNcut is a differential version of Normalized cut \cite{Shi:IEEE2000-Ncut}, hence it can potentially be used to replace the conventional spectral clustering in a variety of applications, especially when facing the clustering task with data of ultra large-scale.
Nevertheless, we also note that in our NeuNcut the feature and the affinity are still assumed to be given and fixed, thus it will be a worthwhile future work to develop a unified framework for jointly learning both the feature and the affinity. In addition, we have observed empirically that our NeuNcut enjoys a good generalization ability to out-of-sample data,
thus it is also an attempting future work to establish the relevant theoretical guarantee.

\section*{Acknowledgment} W. He, C.-G Li and J. Guo are supported by the National Natural Science Foundation of China under Grant 61876022.

{
\small
\bibliographystyle{elsarticle-num}
\bibliography{main}

\begin{thebibliography}{10}
\expandafter\ifx\csname url\endcsname\relax
  \def\url#1{\texttt{#1}}\fi
\expandafter\ifx\csname urlprefix\endcsname\relax\def\urlprefix{URL }\fi
\expandafter\ifx\csname href\endcsname\relax
  \def\href#1#2{#2} \def\path#1{#1}\fi

\bibitem{vonLuxburg:StatComp07}
U.~von Luxburg, A tutorial on spectral clustering, Statistics and Computing 17~(4) (2007) 395--416.

\bibitem{Filippone:PR2008-Survey}
M.~Filippone, F.~Camastra, F.~Masulli, S.~Rovetta, A survey of kernel and spectral methods for clustering, Pattern Recognition 41~(1) (2008) 176--190.

\bibitem{Ding:PR2024-graph_survey}
L.~Ding, C.~Li, D.~Jin, S.~Ding, Survey of spectral clustering based on graph theory, Pattern Recognition 151 (2024) 110366.

\bibitem{Chan:IDAC1993-Rcut}
P.~K. Chan, M.~D. Schlag, J.~Y. Zien, Spectral k-way ratio-cut partitioning and clustering, in: Proceedings of the 30th International Design Automation Conference, 1993, pp. 749--754.

\bibitem{Shi:IEEE2000-Ncut}
J.~Shi, J.~Malik, Normalized cuts and image segmentation, IEEE Transactions on Pattern Analysis and Machine Intelligence 22~(8) (2000) 888--905.

\bibitem{Ding:ICDM2001-Minmax}
C.~H. Ding, X.~He, H.~Zha, M.~Gu, H.~D. Simon, A min-max cut algorithm for graph partitioning and data clustering, in: Proceedings of the IEEE International Conference on Data Mining, 2001, pp. 107--114.

\bibitem{Zhu:PR2020-one}
X.~Zhu, Y.~Zhu, W.~Zheng, Spectral rotation for deep one-step clustering, Pattern Recognition 105 (2020) 107175.

\bibitem{Fowlkes:2004-Nystrom}
C.~Fowlkes, S.~Belongie, F.~Chung, J.~Malik, Spectral grouping using the nystr{\"o}m method, IEEE Transactions on Pattern Analysis and Machine Intelligence 26~(2) (2004) 214--225.

\bibitem{Yang:PR2023-RESKM}
G.~Yang, S.~Deng, X.~Chen, C.~Chen, Y.~Yang, Z.~Gong, Z.~Hao, Reskm: A general framework to accelerate large-scale spectral clustering, Pattern Recognition 137 (2023) 109275.

\bibitem{Yang:PR2020-Fast}
X.~Yang, W.~Yu, R.~Wang, G.~Zhang, F.~Nie, Fast spectral clustering learning with hierarchical bipartite graph for large-scale data, Pattern Recognition Letters 130 (2020) 345--352.

\bibitem{Shaham:ICLR18}
U.~Shaham, K.~P. Stanton, H.~Li, R.~Basri, B.~Nadler, Y.~Kluger, Spectralnet: Spectral clustering using deep neural networks, in: Proceedings of the 6th International Conference on Learning Representations, 2018.

\bibitem{Chen:MSML2022-SpecNet2}
Z.~Chen, Y.~Li, X.~Cheng, Specnet2: Orthogonalization-free spectral embedding by neural networks, in: Proceedings of The Mathematical and Scientific Machine Learning Conference, Vol. 190, 2022, pp. 33--48.

\bibitem{Boutsidis:PR2008-SVD}
C.~Boutsidis, E.~Gallopoulos, Svd based initialization: A head start for nonnegative matrix factorization, Pattern Recognition 41~(4) (2008) 1350--1362.

\bibitem{Lu:PR2014-Non}
H.~Lu, Z.~Fu, X.~Shu, Non-negative and sparse spectral clustering, Pattern Recognition 47~(1) (2014) 418--426.

\bibitem{Shang:PR2016-Global}
R.~Shang, Z.~Zhang, L.~Jiao, W.~Wang, S.~Yang, Global discriminative-based nonnegative spectral clustering, Pattern Recognition 55 (2016) 172--182.

\bibitem{Huang:IEEE2019-Ultra}
D.~Huang, C.-D. Wang, J.-S. Wu, J.-H. Lai, C.-K. Kwoh, Ultra-scalable spectral clustering and ensemble clustering, IEEE Transactions on Knowledge and Data Engineering 32~(6) (2019) 1212--1226.

\bibitem{Jiang:IJCAI2016-VaDE}
Z.~Jiang, Y.~Zheng, H.~Tan, B.~Tang, H.~Zhou, Variational deep embedding: An unsupervised and generative approach to clustering, in: Proceedings of the International Joint Conference on Artificial Intelligence, 2017, pp. 1965--1972.

\bibitem{Ghasedi:ICCV17-DEPICT}
K.~Ghasedi~Dizaji, A.~Herandi, C.~Deng, W.~Cai, H.~Huang, Deep clustering via joint convolutional autoencoder embedding and relative entropy minimization, in: {IEEE} International Conference on Computer Vision, 2017, pp. 5736--5745.

\bibitem{Van:ECCV2020-SCAN}
W.~Van~Gansbeke, S.~Vandenhende, S.~Georgoulis, M.~Proesmans, L.~Van~Gool, Scan: Learning to classify images without labels, in: European Conference on Computer Vision, 2020, pp. 268--285.

\bibitem{Caron:ECCV2018}
M.~Caron, P.~Bojanowski, A.~Joulin, M.~Douze, Deep clustering for unsupervised learning of visual features, in: European Conference on Computer Vision, 2018, pp. 132--149.

\bibitem{Fan:NIPS2022-AutoSC}
J.~Fan, Y.~Tu, Z.~Zhang, M.~Zhao, H.~Zhang, A simple approach to automated spectral clustering, in: Advances in Neural Information Processing Systems, Vol.~35, 2022, pp. 9907--9921.

\bibitem{Streicher:CVPR2023-BaSiS}
O.~Streicher, I.~Cohen, G.~Gilboa, Basis: Batch aligned spectral embedding space, in: Proceedings of the IEEE Conference on Computer Vision and Pattern Recognition, 2023, pp. 10396--10405.

\bibitem{Nazi:arXiv2019-CNC}
A.~Nazi, W.~Hang, A.~Goldie, S.~Ravi, A.~Mirhoseini, Generalized clustering by learning to optimize expected normalized cuts, arXiv preprint arXiv:1910.07623 (2019).

\bibitem{Zhang:CVPR21-SENet}
S.~Zhang, C.~You, R.~Vidal, C.-G. Li, Learning a self-expressive network for subspace clustering, in: Proceedings of the IEEE Conference on Computer Vision and Pattern Recognition, 2021, pp. 12393--12403.

\bibitem{Belkin:NIPS06-laplacian_theory}
M.~Belkin, P.~Niyogi, Convergence of laplacian eigenmaps, Advances in Neural Information Processing Systems (2006) 129--136.

\bibitem{Belkin:JCSC08-laplacian_theory}
M.~Belkin, P.~Niyogi, Towards a theoretical foundation for laplacian-based manifold methods, Journal of Computer and System Sciences 74~(8) (2008) 1289--1308.

\bibitem{Elhamifar:CVPR09}
E.~Elhamifar, R.~Vidal, Sparse subspace clustering, in: {IEEE} Conference on Computer Vision and Pattern Recognition, 2009, pp. 2790--2797.

\bibitem{You:CVPR16-EnSC}
C.~You, C.-G. Li, D.~Robinson, R.~Vidal, Oracle based active set algorithm for scalable elastic net subspace clustering, in: {IEEE} Conference on Computer Vision and Pattern Recognition, 2016, pp. 3928--3937.

\bibitem{Lecun:pe1998}
Y.~{Lecun}, L.~{Bottou}, Y.~{Bengio}, P.~{Haffner}, Gradient-based learning applied to document recognition, Proceedings of the IEEE 86~(11) (1998) 2278--2324.

\bibitem{Xiao:FashionMNIST19}
H.~Xiao, K.~Rasul, R.~Vollgraf, Fashion-mnist: a novel image dataset for benchmarking machine learning algorithms, arXiv preprint arXiv: 1708.07747 (2019).

\bibitem{Cohen:IJCNN2017-Emnist}
G.~Cohen, S.~Afshar, J.~Tapson, A.~Van~Schaik, Emnist: Extending mnist to handwritten letters, in: Proceedings of the IEEE International Joint Conference on Neural Networks, 2017, pp. 2921--2926.

\bibitem{You:ECCV18}
C.~You, C.~Li, D.~Robinson, R.~Vidal, A scalable exemplar-based subspace clustering algorithm for class-imbalanced data, in: European Conference on Computer Vision, 2018, pp. 68--85.

\bibitem{Krizhevsky:CIFAR2009}
A.~Krizhevsky, G.~Hinton, et~al., Learning multiple layers of features from tiny images, Technical Report TR-2009, University of Toronto, Toronto (2009).

\bibitem{Loosli:2007-Infi}
G.~Loosli, S.~Canu, L.~Bottou, Training invariant support vector machines using selective sampling, in: Large Scale Kernel Machines, MIT press, 2007, pp. 301--320.

\bibitem{Deng:CVPR2009-imagenet}
J.~Deng, W.~Dong, R.~Socher, L.-J. Li, K.~Li, L.~Fei-Fei, Imagenet: A large-scale hierarchical image database, in: IEEE conference on computer vision and pattern recognition, 2009, pp. 248--255.

\bibitem{Bruna2013}
J.~Bruna, S.~Mallat, Invariant scattering convolution networks, {IEEE} Transactions on Pattern Analysis and Machine Intelligence 35~(8) (2013) 1872--1886.

\bibitem{Yu:NIPS20}
Y.~Yu, K.~H.~R. Chan, C.~You, C.~Song, Y.~Ma, Learning diverse and discriminative representations via the principle of maximal coding rate reduction, in: Advances in Neural Information Processing Systems, 2020, pp. 9422--9434.

\bibitem{Radford:ICML2021-CLIP}
A.~Radford, J.~W. Kim, C.~Hallacy, A.~Ramesh, G.~Goh, S.~Agarwal, G.~Sastry, A.~Askell, P.~Mishkin, J.~Clark, G.~Krueger, I.~Sutskever, Learning transferable visual models from natural language supervision, in: Proceedings of the International Conference on Machine Learning, 2021, pp. 8748--8763.

\bibitem{MacQueen-1967}
J.~MacQueen, Some methods for classification and analysis of multivariate observations, in: Proceedings of the Fifth Berkeley Symposium on Mathematical Statistics and Probability, 1967, pp. 281--297.

\bibitem{Cai:CVPR2022-EDESC}
J.~Cai, J.~Fan, W.~Guo, S.~Wang, Y.~Zhang, Z.~Zhang, Efficient deep embedded subspace clustering, in: Proceedings of the IEEE Conference on Computer Vision and Pattern Recognition, 2022, pp. 1--10.

\end{thebibliography}
}

\end{document}